%% file: neurips_2023.tex
\documentclass{article}

\usepackage[nonatbib,final]{neurips_2023}

\usepackage[numbers]{natbib}
\usepackage[utf8]{inputenc} 
\usepackage[T1]{fontenc}    
\usepackage{hyperref}       
\usepackage{url}            
\usepackage{booktabs}       
\usepackage{amsfonts}       
\usepackage{nicefrac}       
\usepackage{microtype}      
\usepackage{xcolor}         

\usepackage{amssymb}
\usepackage{mathtools}
\usepackage{latexsym}
\usepackage{longtable}
\usepackage{makecell}
\usepackage{multirow}
\usepackage{tikz}
\usepackage{times}
\usepackage{soul}
\usepackage[small]{caption}
\usepackage{graphicx}
\usepackage{amsmath}
\usepackage{amsthm}
\usepackage{algorithm}
\usepackage{algorithmic}
\usepackage{bm}
\usepackage{comment}
\usepackage{stfloats}
\usepackage{subfig}
\newtheorem{thm}{Theorem}[section]

\newtheorem{lem}[thm]{Lemma}

\theoremstyle{definition}
\newtheorem{defn}[thm]{Definition}
\newtheorem{assu}[thm]{Assumption}
\theoremstyle{remark}

\usepackage{slashbox}

\usepackage[textsize=tiny]{todonotes}

\title{Understanding and Improving Ensemble \\ Adversarial Defense}

\author{
  Yian Deng
\\
  Department of Computer Science\\
  The University of Manchester\\
   Manchester, UK, M13 9PL \\
  \texttt{yian.deng@manchester.ac.uk} \\
   \And
   Tingting Mu \\
   Department of Computer Science\\
   The University of Manchester\\
   Manchester, UK, M13 9PL \\
   \texttt{tingting.mu@manchester.ac.uk} \\
}

\begin{document}

\maketitle

\begin{abstract}	

The strategy of ensemble  has become popular in adversarial defense, which trains multiple base classifiers to defend against adversarial attacks in a cooperative manner.  
Despite the empirical success, theoretical explanations on why an ensemble of adversarially trained classifiers is more robust than single ones remain unclear.
To fill in this gap, we  develop a new error theory dedicated to understanding ensemble adversarial defense,  demonstrating a provable 0-1 loss reduction on challenging sample sets  in adversarial defense scenarios.
Guided by this theory, we propose an effective approach to  improve ensemble adversarial defense, named interactive global adversarial training (iGAT). 
The proposal includes (1) a probabilistic distributing rule that selectively allocates to different base  classifiers  adversarial examples that are globally challenging to the ensemble, and (2) a regularization term to rescue the severest weaknesses of the base classifiers.  
Being tested over various existing ensemble adversarial defense techniques,    iGAT is capable of boosting their performance by up to $17\%$  evaluated using  CIFAR10 and CIFAR100 datasets under both   white-box and black-box attacks.
\end{abstract}

\input{1_introduction}

\input{2_notations_and_definitions}

\input{3_theory}
\input{4_proposed_method}

\input{5_experiments_and_results}

\section{Conclusion, Limitation and Future Work}

We investigate the challenging and crucial problem of defending against adversarial attacks  in the input space of a neural network, with the goal of enhancing ensemble robustness against such attacks while without sacrificing much the natural accuracy. 
We  have provided  a formal justification of the advantage of ensemble adversarial defence and proposed an effective algorithmic improvement, bridging the gap between theoretical and practical studies.
Specifically, we have proven a decrease in empirical 0-1 loss calculated on data samples challenging to classify, which is constructed to simulate the adversarial attack and defence scenario, under neural network assumptions that are feasible in practice.
Also, we   have proposed the iGAT  approach,  applicable to any ensemble adversarial defense technique for improvement. 
It is supported by (1) a probabilistic distributing rule for selectively allocating global adversarial examples to train base classifiers,  and (2) a  regularization penalty for addressing  vulnerabilities across all base classifiers.
We have conducted thorough evaluations and ablation studies using the CIFAR-10 and CIFAR-100 datasets, demonstrating effectiveness of the key designs of iGAT. Satisfactory performance improvements up to $17\%$ have been achieved by iGAT.

However, there is limitation in our work. For instance, our theoretical result is developed for only two base MLPs. 
We are   in progress of broadening the scope of Theorem \ref{main_theorem} by further relaxing the neural network assumptions, researching model architectures beyond MLPs and beyond the average/max combiners, and more importantly generalizing the theory  to more than two base classifiers. 
Additionally,  we are keen to enrich our evaluations using large-scale datasets, e.g., ImageNet.
So far, we focus on exploiting curvature information of the loss landscapes to understand adversarial robustness. In the future, it would be interesting to explore richer geometric information to improve the understanding.
Despite the research success, a potential negative societal impact of our work is that it may prompt illegal attackers to develop new attack methods once they become aware of the underlying mechanism behind the ensemble cooperation.


\section*{Acknowledgments}
We thank the five NeurIPS reviewers for their very insightful and useful comments that help improve the paper draft. We also thank ``The University of Manchester - China Scholarship Council Joint Scholarship'' for funding Yian's PhD research.
\bibliographystyle{plainnat}
\bibliography{ref}

\newpage
\section*{Appendix}

\input{supp_UsedExisting}

\input{supp_lemma}

\input{supp_theorem}
\input{supp_Toy}

\input{supp_ExtraResults}

\end{document}

%% file: 1_introduction.tex
\section{Introduction}

Many contemporary machine learning models, particularly the end-to-end ones based on deep neural networks, admit vulnerabilities to small perturbations in the input feature space. 
For instance, in computer vision applications, a  minor change of image pixels computed by an algorithm can manipulate the classification results to produce undesired predictions, but these pixel changes can be imperceptible to human eyes.  
Such malicious perturbations are referred to as adversarial attacks. 
These can result in severe incidents, e.g., medical misdiagnose caused by unauthorized perturbations in medical imaging~\cite{wang2019security},  and  wrong actions taken by autonomous vehicles caused by crafted traffic images~\cite{ren2019security}.

The capability of a machine learning model to defend  adversarial attacks  is referred to as adversarial robustness. 
A formal way to  quantify such robustness is through an adversarial risk,  which can be intuitively understood as the expectation of the worst-scenario loss computed within a local neighborhood region around a naturally sampled data example~\cite{uesato2018adversarial}. 
It is formulated as below:
\begin{equation}
	R_{adv}(\mathbf{f})=\mathbb{E}_{(\mathbf{x}, y)\sim \mathcal{D}}\left[\max_{\mathbf{z}\in\mathcal{B}(\mathbf{x})} \ell\left( \hat{y}\left(\mathbf{f}(\mathbf{z})\right), y\right)\right],
\end{equation}
where $\mathcal{B}(\mathbf{x})$ denotes a  local neighborhood region around the example $\mathbf{x}$ sampled from a natural data distribution $ \mathcal{D}$. The neighbourhood is usually defined based on a selected norm, e.g., a region containing any $\mathbf{z}$ satisfying $  \left\|\mathbf{z}-\mathbf{x}\right\|\le \varepsilon$ for a constant $\varepsilon>0$.   
In general, $\ell(\hat{y}, y)$ can be any loss function quantifying the difference between the predicted output $\hat{y}$ and the  ground-truth output $y$.
A learning process that minimizes the adversarial risk is referred to as  adversarially robust learning.  
In the classification context, a way to empirically approximate the adversarial robustness is through  a classification error computed using a set of adversarial examples generated by applying  adversarial attacks  to the model~\cite{ijcai2021p591}. 
Here, an adversarial attack refers to an algorithm that usually builds on an optimization strategy, and it produces examples perturbed  by a certain strength  so that a machine learning model returns the most erroneous output for these examples~\cite{chakraborty2021survey,lowd2005adversarial}.   

On the  defense  side, there is a rich amount of techniques  proposed to improve the model robustness~\cite{chakraborty2021survey,akhtar2021advances}. 
However,  a robust neural network enhanced by a   defense technique would mostly result in a reduced accuracy for classifying the natural examples~\cite{pmlrv97zhang19p}. 
A  relation between the adversarial robustness and the standard accuracy has  been formally proved by~\citet{TsiprasSETM19Robustness}:  
\textit{``Any classifier that attains at least $1-\delta$ standard accuracy on a dataset $\mathcal{D}$ has a robust accuracy at most $\delta p/(1-p)$ against an $L_\infty$-bounded adversary with $\varepsilon \ge 2\eta$.''}, where  $\varepsilon$ is the $L_\infty$ bound indicating the attack strength, while $\eta$ is sufficiently large and $p\geq 0.5$ and they both are used for data generation.
This statement presents a tradeoff between the adversarial robustness and natural accuracy. 
There has been a consistent effort invested to  mitigate such tradeoff by developing defense techniques to improve robust accuracy without sacrificing much the natural accuracy.

Ensemble  defense has recently arisen as a new  state of the art~\cite{ijcai2021p591,lu2022ensemble}. 
The core idea is to  train strategically multiple base classifiers  to defend the attack, with the underlying motivation of improving the statistical stability and   cooperation between base models.  
Existing effort on ensemble defense has mostly been focused on demonstrating performance success for different algorithmic approaches. 
It is assumed that  \emph{training and combining multiple base models can defend better adversarial attacks as compared to training a  single model}. 
Although being supported by empirical success, there is few research that provides rigorous  understanding to why this is the case in general.
Existing results on  analyzing generalized error for ensemble estimators  mostly compare the error of the ensemble  and the averaged error of the base models, through, for instance, decomposition strategies that divide the error term into bias, variance, co-variance, noise  and/or diversity terms~\cite{wood2023unified,brown2005managing,ueda1996generalization}.  
It is not straightforward  to  extend such results to compare the error of an ensemble model and the error of a model without an ensemble structure.

To address this gap, we   develop a new error  theory (Theorem \ref{main_theorem}) dedicated to  understanding ensemble adversarial defense.  
The main challenge in mitigating the tradeoff between the adversarial robustness and natural accuracy comes from the fact that  adversarial defence techniques can reduce the classifier's capacity of handling weakly separable example pairs that are close to each other but from different classes.
To analyse how ensemble helps address this particular challenge, we derive a provable error reduction by changing from using one neural network to  an ensemble of two neural networks through either the average or max combiner of their prediction outputs.

Although ensemble defense can improve  the overall adversarial robustness, its base models can still be  fooled individually in some input  subspaces,  due to the nature of the collaborative design. 
Another contribution we make is the proposal of a simple but effective way to improve each base model by considering the adversarial examples generated by their  ensemble  accompanied by a regularization term that is designed to recuse the worst base model.
We experiment with the proposed enhancement by applying it to improve four state-of-the-art ensemble  adversarial  defense techniques. 
Satisfactory performance improvement has been  observed when being evaluated using CIFAR-10 and CIFAR-100 datasets~\cite{Zagoruyko2016WRN}  under both white-box and black-box attacks.

\section{Related Work} \label{related}

\textbf{Adversarial Attack:} Typical attack techniques include white-box and black-box attacks.  
The white-box adversary has access to the model, e.g., the model parameters, gradients and formulation, etc, while the black-box  adversary has limited knowledge of the model, e.g., knowing only the model output, and therefore is closer to the real-world attack setting~\cite{bhambri2019survey}.
Representative white-box attacks include Deepfool~\cite{moosavi2016deepfool}   and the projected gradient descent (PGD)~\cite{madryMSTV18towards}. 
The  fast gradient sign method (FGSM)~\cite{goodfellow2014explaining} is a simplified one-iteration version of PGD. 
The momentum iterative method (MIM)~\cite{dong2018boosting} improves the FGSM attack by introducing the momentum term.
Another two commonly-used and cutting-edge white-box attacks are the Carlini-Wagner (CW) attack~\cite{carlini2017towards} that computes the perturbation by optimizing the perturbation scales, and the Jacobian-based saliency map attack (JSMA)~\cite{papernokey6limitations} that perturbs only one pixel in each iteration.
Representative black-box attacks include the square attack (SA)~\cite{andriushchenko2020square}, the SignHunter attack~\cite{al2020sign} and the simple black-box attack (SimBA)~\cite{guo2019simple}, among which SA can generate stronger attacks but is more computationally expensive.   
A thorough survey on adversarial attacks is provided by~\citet{akhtar2021advances}.  
AutoAttack~\cite{croce2020reliable} encapsulates a selected set of strong white and black-box attacks and is considered as the  state-of-the-art attack tool. 
In general, when evaluating the adversarial robustness of a deep learning model, it should be tested under both the white and black-box attacks.

\textbf{Adversarial Defense:} Adversarial training is the most straightforward and commonly used defense technique for improving adversarial robustness. 
It works by  simply expanding the training data with additional adversarial examples.   
The initial idea was firstly adopted by~\citet{SzegedyZSBEGF13} to train robust neural networks for classification, later on, \citet{madryMSTV18towards} used  adversarial examples generated by the PGD attack  for improvement. 
Overall, this strategy is very adaptive and can be used to defend any attack, but it is on the expense of  consuming more training examples.
It has been empirically observed that  adversarial training  can reduce  loss curvatures in both the input  and model parameter spaces~\cite{FreemanB17Topology,fawzi2018empirical}, and this reduces the adversarial robustness gap between the training and testing data~\cite{pmlrv97zhang19p,wu2020adversarial}. 
These findings motivate the development of a series of regularization techniques for adversarial defense, by explicitly regularizing the curvature or other relevant geometric characteristics of the loss function. 
Many of these techniques can avoid augmenting the training set and  are computationally cheaper. 
Typical regularization techniques that attempt to flatten the loss surface in the input space include the curvature regularization (CURE)~\cite{moosavi2019robustness}, local linearity regularization (LLR)~\cite{qin2019adversarial}, input gradient regularization~\cite{najafi2019robustness} and Lipschitz regularization~\cite{virmaux2018lipschitz}. 
There are also  techniques to flatten the loss in the model parameter space, such as TRADES~\cite{pmlrv97zhang19p}, misclassification
aware adversarial training (MART)~\cite{wang2019improving}, robust self-training (RST)~\cite{carmon2019unlabeled,raghunathan2020understanding}, adversarial weight perturbation (AWP)~\cite{wu2020adversarial}, and HAT~\cite{wu2021wider}, etc.
Alternative to adversarial training and regularization,  other defense strategies include pruning~\cite{sehwag2020hydra}, pre-training~\cite{hendrycks2019pretraining}, feature denoising~\cite{xie2019feature}, domain adaptation~\cite{SongHWH19Improve} 
and ensemble defense~\cite{lu2022ensemble,tramer2017ensemble}, etc.  
\citet{croce2020robustbench} has provided a summary of  recent advances.

\textbf{Ensemble Adversarial Defense:}  Recently, ensemble has been actively used in adversarial defense, showing promising results. 
One core spirit behind ensemble is to encourage diversity between base models in order to achieve an improved ensemble prediction ~\cite{wood2023unified,brown2005managing,ueda1996generalization}.
Therefore,  the advances  are mostly focused on designing effective ensemble diversity losses in training to improve adversarial robustness. 
For instance, the adaptive diversity promoting (ADP) method~\cite{pang2019improving} uses Shannon entropy for uncertainty regularization and a geometric diversity for measuring the difference between the   predictions made by base classifiers.  
The transferability reduced smooth (TRS) method~\cite{yang2021trs}   formulates the diversity term based on cosine similarities between the loss gradients of base models, meanwhile increases the model smoothness via an $l_2$-regularization of these gradients during training. 
The similar idea of  exploiting loss gradients was proposed earlier in  the gradient alignment loss (GAL) method~\cite{kariyappa2019improving}.  
The conditional label dependency learning (CLDL) method~\cite{wang2023adversarial} is a latest improvement  over GAL  and TRS, which measures the diversity using both the predictions and loss gradients of the base models.  
However, the ensemble nature of encouraging diversity can cause vulnerability for some base models over certain subsets of adversarial examples~\cite{moosavi2019robustness}.  
In practice, this can  limit the overall robustness of the ensemble when the other base models are not strong enough to correct the weak ones. 
To address this,  the diversifying vulnerabilities for enhanced robust generation of ensembles (DVERGE)~\cite{yang2020dverge}  proposes a vulnerability diversity to encourage each base model to be robust particularly to the other base models' weaknesses.  
The latest development for improving ensemble defense, known as synergy-of-experts  (SoE)~\cite{cui2022synergy}, follows a different research path. 
For each input, it  adaptively selects a base model with the largest confidence to make the final prediction instead of combining all, for which the supporting algorithm and theory have been developed.
Some surveys on ensemble adversarial attacks and defense  can be found in \citet{he2022revisiting,lu2022ensemble}.

%% file: 2_notations_and_definitions.tex
\section{Notations and Preliminaries}
\label{notation}

Bold capital  and lower-case letters, e.g., $\mathbf{X}$ and $\mathbf{x}$, denote matrices and  vectors, respectively, while lower-case letters, e.g., $x$, denote scalars. 
The $i$-th row and column of a matrix $\mathbf{X}$ are denoted by $\mathbf{x}_i$ and $\mathbf{x}^{(i)}$,  respectively, while $x_{i,j}$ and $x_i$  the elements of $\mathbf{X}$ and $\mathbf{x}$. 
A  classification  dataset $D = \left\{\left(\mathbf{x}_i, y_i\right)\right\}_{i=1}^n$ includes $n$ examples, which are referred to as \textit{natural examples},  with $\mathbf{x}_i \in \mathcal{X} \subset \mathbb{R}^d$ (feature vector) and  $y_i\in [C]=\{1,2,\ldots, C\}$ (class label).  
We sometimes express the label of an example as  $y(\mathbf{x})$ or $y_{\mathbf{x}}$.
Storing  $\mathbf{x}_i$ as a row of   $\mathbf{X} \in \mathbb{R}^{n\times d}$ and   $y_i$  an element of  $\mathbf{y} \in \mathbb{R}^n$, we also denote this dataset by $D= (\mathbf{X}, \mathbf{y})$.
The  classifier  $\mathbf{f}:\mathcal{X}\to [0,1]^C$ outputs  class probabilities  usually computed by a softmax function. 
Given  the computed probability $f_c$ for the $c$-th class,   $ \hat{y}_{\mathbf{f}}(\mathbf{x}) = \arg\max_{c\in [C]} f_c(\mathbf{x}) $ predicts the class.
For a neural network,  we denote by $\mathbf{W}^{(l)}$ the weight matrix connecting the $l$-th and the $(l-1)$-th layers and by $w^{(l)}_{i,j}$  its $ij$-th element.
The $L_2$-norm $\|\cdot\|_2$ is used to compute the vector length, while the $L_\infty$-norm   $\|\cdot\|_\infty$ to generate adversarial attacks. Concatenation of two sets is denoted by the symbol $\cup$.

We focus on classification by minimizing a classification loss  $\ell(\mathbf{f}(\mathbf{x}), y_{\mathbf{x}})$, and adapt it to $\ell(\mathbf{f}(\mathbf{X}), \mathbf{y})$ for the whole dataset. 
Also, we use $\ell_{CE} $ to emphasize the cross-entropy loss.
The loss gradient  is $\nabla  \ell (\mathbf{f}(\mathbf{x}), y_{\mathbf{x}})=\frac{\partial \ell(\mathbf{f}(\mathbf{x}), y_{\mathbf{x}})}{\partial \mathbf{x}}$. 
A cheap way to estimate the loss curvature is  by finite difference approximation~\cite{moosavi2019robustness}, e.g.,  the following curvature  measure based on $L_2$-norm:
\begin{equation}
	\label{app_cur}
	\lambda_{\mathbf{f}}(\mathbf{x}, \bm\delta)=\frac{ \left \|  \nabla \ell (\mathbf{f}\left(\mathbf{x} +\bm\delta\right), y_{\mathbf{x}})- \nabla \ell (\mathbf{f}(\mathbf{x}), y_{\mathbf{x}}) \right \|_2}{ \|\bm\delta\|_2},
\end{equation}
where $\bm\delta \in \mathbb{R}^d$ is a perturbation. It measures how a surface bends at a point by different amounts in different directions. 
An adversarial example  $\tilde{\mathbf{x}} = \bm\phi(\mathbf{f}, \mathbf{x}, A)$ is generated by attacking  the classifier $\mathbf{f}$ using an attack algorithm $A$ on a natural example $\mathbf{x}$. 
It is further adapted to  $\tilde{\mathbf{X}} = \bm\phi(\mathbf{f}, \mathbf{X}, A)$ for the set of adversarial examples each generated from a natural example in $\mathbf{X}$. 
The quantity $\bm\delta(\mathbf{f}, \mathbf{x}, A)$ $=\bm \phi(\mathbf{f}, \mathbf{x}, A)-\mathbf{x}$ is referred to as the \textit{adversarial perturbation} of $\mathbf{x}$, simplified to $\bm\delta_{\mathbf{x}} = \tilde{\mathbf x} -\mathbf{x}$. 
To control the perturbation strength, we restrict  $\left\|\bm\delta_{\mathbf{x}} \right\|_\infty\le\varepsilon$ for some $\varepsilon>0$, which results in the following adversarial example formulation, as
\begin{equation}
	\bm \phi_{\varepsilon}(\mathbf{f}, \mathbf{x}, A)=\min(\max(\bm\phi(\mathbf{f}, \mathbf{x}, A), \mathbf{x}-\varepsilon), \mathbf{x}+\varepsilon),
\end{equation}
where both $\min(\cdot, \cdot )$ and $\max(\cdot,\cdot)$ are element-wise operators comparing their inputs.

%% file: 3_theory.tex
\section{An Error Theory for Adversarial Ensemble Defense} \label{sec_theory}

In adversarial ensemble defense, a widely accepted research hypothesis is that training and combining multiple base classifiers can improve adversarial defense as compared to training a  single classifier. 
However, this hypothesis is mostly supported by empirical successes and there is a lack of formal theoretical justification.
In this work, we seek  theoretical evidence, proving that, when using multilayer perceptrons (MLPs) for classification,  classification error reduces when applying adversarial defence to the base MLPs of an ensemble as compared to a single MLP, under  assumptions feasible in practice. 
The following theorem formalizes our main result.
\begin{thm}
Suppose $\mathbf{h},\mathbf{h}^0, \mathbf{h}^1\in \mathcal{H}:\mathcal{X}\to [0,1]^C$ are  $C$-class $L$-layer MLPs satisfying  Assumption \ref{MLP_assu}.
Given a dataset  $D=\left\{( \mathbf{x}_i, y_i)\right\}_{i=1}^n$,  construct an ambiguous pair set $A(D)$ by Definition \ref{ass_borderline_data}.  Assume $\mathbf{h},\mathbf{h}^0, \mathbf{h}^1$ are acceptable classifiers for $A(D)$ by Assumption \ref{assu_simple_classification}.  Given a classifier $\mathbf{f}\in \mathcal{H}:\mathcal{X}\to\mathbb{R}^C$ and a dataset $D$, assess its classification error by 0-1 loss, as
\begin{equation}
	\label{eq_01_loss}
	\hat{\mathcal{R}}_{0/1}(D, \mathbf{f})=\frac{1}{|D|}\sum_{\mathbf{x}\in D}{1} \left[f_{y_\mathbf{x}}(\mathbf{x})<\max_{c\ne y_\mathbf{x}}f_c(\mathbf{x})\right],\\
\end{equation} 
where $1[\textmd{true}] =1$ while  $1[\textmd{false}] =0$.
For an ensemble $\mathbf{h}^{(0,1)}_e$ of   two base MLPs $\mathbf{h}^0$ and $\mathbf{h}^1$ through either an average or a max combiner, i.e., $\mathbf{h}^{(0,1)}_e= \frac{1}{2}(\mathbf{h}^0 + \mathbf{h}^1)$ or $\mathbf{h}^{(0,1)}_e=\max(\mathbf{h}^0, \mathbf{h}^1)$, it has a lower empirical 0-1 loss than a single MLP for classifying ambiguous examples, such as
\begin{equation}
\label{eq_thm_1}
 	 \mathbb{E}_{a\sim A(D)}\mathbb{E}_{\mathbf{h}^0,\mathbf{h}^1\in\mathcal{H}}\left[\hat{\mathcal{R}}_{0/1}\left(a,  \mathbf{h}^{(0,1)}_e \right)\right] <  \mathbb{E}_{a\sim A(D)}\mathbb{E}_{\mathbf{h}\in\mathcal{H}}\left[\hat{\mathcal{R}}_{0/1}\left(a,\mathbf{h}\right)\right].
\end{equation}
	\label{main_theorem}
\end{thm} 

We prove the result for  MLPs satisfying the following assumption.
\begin{assu} [\textbf{MLP Requirement}]
\label{MLP_assu}
 Suppose a $C$-class $L$-layer MLP $\mathbf{h}: \mathbb{R}^d\rightarrow [0,1]^C$  expressed iteratively by 
\begin{align}
\label{mlp1}
\mathbf{a}^{(0)}(\mathbf{x}) &= \mathbf{x}, \\
\mathbf{a}^{(l)} (\mathbf{x}) &= \sigma\left(\mathbf{W}^{(l)}\mathbf{a}^{(l-1)} (\mathbf{x})\right),  l =1,2,...,L-1, \\
\mathbf{a}^{(L)} (\mathbf{x}) &=\mathbf{W}^{(L)}\mathbf{a}^{(L-1)} (\mathbf{x})= \mathbf{z}(\mathbf{x}), \\
\label{mlp2}
\mathbf{h}(\mathbf{x}) &= \textmd{softmax}(\mathbf{z}(\mathbf{x})),
\end{align}
where $\sigma(\cdot)$ is the activation function applied element-wise, the representation vector $\mathbf{z}(\mathbf{x})  \in \mathbb{R}^C$ returned by the $L$-th layer  is fed into  the prediction layer building upon the softmax function. 
Let $w_{s_{l+1},s_{l}}^{(l)}$  denote the network weight connecting the $s_{l}$-th neuron in the $l$-th  layer and the $s_{l+1}$-th neuron in the $(l+1)$-th layer for $l\in\{1,2\ldots, L\}$.   
Define  a  column vector $\mathbf{p}^{(k)}$ with its $i$-th element computed from the neural network weights and activation derivatives, as  $p^{(k)}_i = \sum_{s_L} \frac{\partial a^{(L-1)}_{s_L}(\mathbf{x})}{\partial x_k} w^{(L)}_{i,s_L} $ for $k=1,2, \ldots d$ and $i=1,2, \ldots C$, also a matrix $\mathbf{P}_\mathbf{h} = \sum_{k=1}^d\mathbf{p}^{(k)}{\mathbf{p}^{(k)}}^T$ and its factorization $\mathbf{P}_{\mathbf{h}} = \mathbf{M}_\mathbf{h}\mathbf{M}_\mathbf{h}^T$ with a full-rank factor matrix $\mathbf{M}_\mathbf{h}$.  
For constants $\tilde{\lambda},B>0$, suppose the following holds for $\mathbf{h}$: 
\begin{enumerate}
\item Its cross-entropy loss curvature measured by   Eq. (\ref{app_cur}) satisfies $\lambda_{\mathbf{h}}(\mathbf{x}, \bm\delta)\leq \tilde{\lambda}$.
\item  The factor matrix satisfies $\left\|\mathbf{M}_{\mathbf{h}} \right\|_2 \leq B_0$ and $\left\|\mathbf{M}_{\mathbf{h}}^{\dagger}\right\|_2 \leq B$, where $\|\cdot\|_2$ denotes the vector induced $l_2$-norm for matrix.
\end{enumerate}
 \end{assu}
We explain the feasibility of the above MLP assumptions in the end of this section.

Although adversarial defense techniques can improve adversarial robustness,  new challenges arise in classifying  examples that are close to each other but from different classes, due to the  flattened loss curvature for reducing the adversarial risk. 
We refer to  a pair of such challenging examples as an \emph{ambiguous  pair}.  Our strategy of proving improved performance for adversarial defense  is to (1) firstly construct  a challenging dataset $A(D)$  comprising samples from these  pairs, which is referred to as an \emph{ambiguous pair set}, and then (2) prove error reduction over $A(D)$.
To start, we provide formal definitions for the ambiguous pair and set.
\begin{defn}[\textbf{Ambiguous Pair}] \label{ass_borderline_data}
Given a dataset  $D=\left\{( \mathbf{x}_i, y_i)\right\}_{i=1}^n$ where $\mathbf{x}_i\in \mathcal{X}$ and $y_i \in [C]$, an \emph{ambiguous pair} contains two examples $a = ((\mathbf{x}_i, y_i), (\mathbf{x}_j , y_j))$   satisfying $y_i \neq y_j$ and 
	\begin{align}
	&\left \|  \mathbf{x}_i - \mathbf{x}_j \right\|_2 \le \frac{1}{J B\sqrt{C\left(\tilde{\lambda}^2 - \xi\right)}},
	\label{eq_ambig_cond}
	\end{align}
where $J>2$ is an adjustable control variable,   $\tilde{\lambda}$, $B$ and $\xi\le \tilde{\lambda}^2$ are constants associated with the MLP  under  Assumption \ref{MLP_assu}.
The \emph{ambiguous pair set} $A(D)$ contains  all the  ambiguous  pairs  existing in $D$, for which $J$ is adjusted such that $A(D)\neq  \emptyset$.
\end{defn}

In Theorem \ref{main_theorem}, we are only interested in classifiers that do not fail too badly on $A(D)$, e.g., having an accuracy level above $42.5\%$.  Comparing poorly performed classifiers is not  very meaningful, also the studied situation is  closer to practical setups where  the starting classifiers for improvement are somewhat acceptable.    Such a preference is formalized by the  following assumption:
\begin{assu} [\textbf{Acceptable Classifier}]
	\label{assu_simple_classification}
	Suppose an acceptable classifier $\mathbf{f}: \mathbb{R}^d\rightarrow [0,1]^C$ does not perform  poorly on the ambiguous pair set $A(D)$ associated with a control variable $J$. This means that,  for any pair $a=((\mathbf{x}_i,  y_i), (\mathbf{x}_j, y_j)) \in A(D)$ and for any example $(\mathbf{x}_i, y_i)$ from the pair, the following  holds: 
	\begin{enumerate}
			\item With a probability $p\geq 42.5\% $, the classifier correctly classifies  $(\mathbf{x}_i, y_i)$ by a sufficiently large  predicted score, i.e.,  $f_{y_i}(\mathbf{x}_i)  \ge0.5 + \frac{1}{J}$, while wrongly classifies the other example $\mathbf{x}_j$ to $y_i$ by a less score,  i.e.,  $f_{y_i}(\mathbf{x}_j)  \le 0.5 + \frac{1}{J}$.	
			\item When the classifier predicts $(\mathbf{x}_i, y_i)$  to class $\hat{y}_i$,  the predicted scores for the other classes excluding $y_i$ are sufficiently small, i.e., $f_{c}(\mathbf{x}_i)\leq \frac{1- f_{\hat{y}_i}(\mathbf{x}_i)}{C-1}$ for $c \neq y_i, \hat{y}_i$.
	\end{enumerate}
\end{assu}
Proof for Theorem \ref{main_theorem} together with a toy illustration example is provided in supplementary material.  

\textbf{Assumption Discussion.}
Assumption \ref{MLP_assu} is feasible in practice. Reduced loss curvature is a natural result from adversarial defense, particularly for  adversarial training and regularization based methods~\cite{fawzi2018empirical,FreemanB17Topology} as mentioned in Section \ref{related}.  Regarding its  second part determined by neural network weights and activation derivatives, common training practices  like weight regularization and normalization help prevent from obtaining overly inflated   elements in $\mathbf{P}_{\mathbf{h}}$, and thus bound  $\left\|\mathbf{M}_{\mathbf{h}}\right\|_2$ and $\left\|\mathbf{M}_{\mathbf{h}}^{\dagger}\right\|_2$. 
Following Definition \ref{ass_borderline_data}, the ambiguous pair $a= ((\mathbf{x}_i, y_i), (\mathbf{x}_j, y_j))$ is constructed to let the classifier struggle with classifying the  neighbouring example, e.g., $(\mathbf{x}_j, y_j)$, when it is able to classify successfully, e.g.,  $(\mathbf{x}_i, y_i)$. 
Consequently, the success of classifying $(\mathbf{x}_i, y_i)$ is mostly accompanied with a  failure of classifying $(\mathbf{x}_j, y_j)$ into  $y_i$, and vice versa.
In Assumption \ref{assu_simple_classification}, for an acceptable classifier, the first part assumes its failure is fairly mild, while the second part  assumes its struggle is  between $y_i$ and $y_j$. 
As shown in our proof of Theorem \ref{main_theorem}, in order for Eq. (\ref{eq_thm_1}) to hold, a polynomial inequality of the probability $p$ needs to be solved, providing a sufficient condition on achieving a reduced ensemble risk, i.e., $p\geq 42.5\% $.
Later, we conduct experiments to examine how well some assumptions are met by adversarially trained classifiers and report the results in supplementary material. 

%% file: 4_proposed_method.tex
\section{iGAT: Improving Ensemble Mechanism} \label{proposed}

Existing ensemble adversarial defense techniques mostly base their design on a framework of combining classification loss and diversity for training. 
The output of each base  classifier contains the probabilities of an example belonging to  the $C$ classes. 
For an input example $\mathbf{x}\in \mathcal{X}$, we denote its output from the $i$-th base  classifier by $\mathbf{h}^i(\mathbf{x})=\left[h_1^i(\mathbf{x})...,h_C^i(\mathbf{x})\right]$ for $i\in[N]$, where $N$  denotes the number of used base classifiers. 
Typical practice for combining base predictions includes  the averaging, i.e.,  $\mathbf{h}(\mathbf{x})= \frac{1}{N}\sum_{i=1}^N\mathbf{h}^i(\mathbf{x})$, or the max operation, i.e.,  $h_j(\mathbf{x}) =   \max_{i\in[N]}  \left(h_j^i(\mathbf{x}) \right)$. 
Without loss of generality, we  denote the combiner by $\mathbf{h}=\mathbf{c}\left(\mathbf{h}^1,...,  \mathbf{h}^{N}\right)$.
To train the base classifiers,  we exemplify  an ensemble loss function using one training example $(\mathbf{x}, y_{\mathbf{x}})$, as below
\begin{equation}
	\label{eq_el}
	L_{E}(\mathbf{x}, y_{\mathbf{x}})= \underbrace{\sum_{i=1}^N \ell(\mathbf{h}^i(\mathbf{x}), y_{\mathbf{x}})}_{\textmd{classification loss}} + \omega \textmd{Reg}\left(\mathbf{h}(\mathbf{x})\right)  + \gamma \textmd{Diversity}(\mathbf{h}^1(\mathbf{x}), \mathbf{h}^2(\mathbf{x}), \ldots,\mathbf{h}^N(\mathbf{x}), y_{\mathbf{x}})),
\end{equation}
where $\omega,\gamma\geq 0$ are hyperparameters. 
An example choice for regularization is the Shannon entropy of the ensemble  $\mathbf{h}(\mathbf{x})$~\cite{pang2019improving}. 
Significant research effort has been invested to diversity design, for which it is optional whether to use the class information in diversity calculation. 
In  the first section of   supplementary material, we briefly explain four ensemble adversarial defense techniques highlighting their loss design strategies. 
These include ADP~\cite{pang2019improving}, CLDL~\cite{wang2023adversarial}, DVERGE~\cite{yang2020dverge} and SoE~\cite{cui2022synergy}, and they are used later in Section \ref{exp} to test our proposed enhancing approach.

Despite the effort in diversity design that encourages better collaboration between base classifiers, it is unavoidable for  some  base classifiers to struggle with classifying examples from certain input subspaces. 
There are  intersected subspaces that all the base classifiers are not good at classifying.
To address this, we propose  an \emph{interactive global adversarial training} (iGAT) approach.
It seeks support from adversarial examples globally generated by the ensemble and distributes these examples to base classifiers with a probabilistic strategy empirically proven effective. 
Additionally, it introduces another regularization term to improve over the  severest weakness of the base classifiers.  
Below we describe our proposal in detail.

\subsection{Distributing Global Adversarial Examples} \label{global_adv}

We aim at improving adversarial robustness over intersected feature subspaces which are hard for all base classifiers to  classify. These regions  can be approximated by  global adversarial examples generated by applying adversarial attacks to the ensemble, which are
\begin{equation}
	(\tilde{\mathbf{X}}, \tilde{\mathbf{y}})= \left( \bm \phi_\varepsilon(\mathbf{c}\left(\mathbf{h}^1,...,  \mathbf{h}^{N}\right),\mathbf{X}, A), \mathbf{y} \right),
	\label{eq_general_attack}
\end{equation}
where  rows of $\tilde{\mathbf{X}}$ store the feature vectors of the generated  adversarial examples.
For instance, the FGSM attack can be used as $A$.
Instead of feeding the same full   set of adversarial examples  to train each base classifier, we distribute different examples to different base classifiers,  to improve performance and to reduce  training time. 
The generated examples are divided into $N$ groups  according to their predicted class probabilities. The $i$-th group   $ \left(\tilde{\mathbf{X}}^i, \tilde{\mathbf{y}}^i\right)$ is used to train the $i$-th base classifier, contributing to its classification loss.

Our core distributing strategy  is to encourage each base classifier to keep improving over regions that they  are  relatively good at classifying. 
This design is motivated by our theoretical result.  
We have  proved in Theory \ref{main_theorem} an error reduction achieved by the ensemble of base MLPs that satisfy the acceptability Assumption \ref{assu_simple_classification}.  
This assumption is partially examined by whether the classifier returns a sufficiently high  prediction score for the correct class or low scores for most of the incorrect classes for some challenging examples. 
By keeping assigning each base classifier new challenging examples  that they  are   relatively good at classifying,  it encourages Assumption \ref{assu_simple_classification} to continue to hold. 
In Section \ref{ablation}, we perform  ablation studies to compare our proposal with a few other distributing strategies,  and the empirical results also verify our design.  
Driven by this strategy, we propose one hard and one soft distributing rule.

\textbf{Hard Distributing Rule:} Given a generated adversarial example $(\tilde{\mathbf{x}}, y) =  \left(\big(\tilde{\mathbf{X}}\big)_k, \tilde{y}_k\right)$,  the following  rule determines which base classifier to assign  it:
\begin{equation}
	\begin{split}
		& \textmd{If }   h^i_{y}(\tilde{\mathbf{x}})>\max_{j\ne i, j\in[N]}h^j_{y}(\tilde{\mathbf{x}}),  \textmd{assign } (\tilde{\mathbf{x}}, y) \textmd{ to }  \left(\tilde{\mathbf{X}}^i, \tilde{\mathbf{y}}^i\right). 
	\end{split}
	\label{eq_split_global_data}
\end{equation}
We refer to it as a hard distributing rule as it simply assigns examples in a deterministic way. The example is assigned to the base classifier that returns the highest predicted probability on its ground truth class.

\textbf{Soft Distributing Rule:} 
A hard assignment like the above can be sensitive to errors. 
Alternatively,  we propose  a soft  distributing rule that utilizes the ranking of the base classifiers based on their prediction performance meanwhile introduces uncertainty.
It builds upon roulette wheel selection~\cite{blickle1996comparison}, which is a commonly used genetic operator in genetic algorithms for selecting promising  candidate solutions.
Firstly, we rank  in descending order the predicted probabilities $\{h^i_{y}(\mathbf{x})\}_{i=1}^N$ by all the base classifiers  for the ground truth class, and let $r_\mathbf{x}\left(\mathbf{h}^i\right) \in [N]$ denote the obtained ranking for the $i$-th base classifier. 
Then, we formulate a ranking-based  score for each base classifier as
\begin{equation}
	\label{eq_biased_sampling}
	p_i = \frac{2^{N-r_\mathbf{x}\left(\mathbf{h}^i\right)}}{\sum_{i\in[N]} 2^{i-1} },
\end{equation}
and it satisfies $\sum_{i\in[N]} p_i =1$. A  more top ranked base classifier has higher  score. 
Next, according to $\{p_i\}_{i=1}^N$, we apply roulette wheel selection and  distribute the example to the selected base classifier. 
Specifically, the selection algorithm constructs $N$ intervals $\{[a_i, b_i]\}_{i=1}^N$ where $a_1 = 0$, $b_1 = p_1$,  also $a_i = b_{i-1}$ and $b_i = a_i + p_i$ for $i=2, 3, \ldots, N$. 
After sampling a number $q\in (0,1]$ following a uniform distribution $q\sim U(0,1)$, check which interval $q$ belongs to. If $a_i < q \leq b_i $,  then the example is  used to train the $i$-th base classifier.
This enables to assign examples based on  ranking but in a probabilistic manner  in order to be more robust to errors.

\subsection{Regularization Against Misclassification} \label{reg_misclass}

We  introduce another regularization term to address the  severest weakness, by minimizing the probability score of the most incorrectly  predicted class by the most erroneous  base classifier. 
Given an input example $(\mathbf{x}, y_{\mathbf{x}})$, the proposed  term is formulated as
\begin{equation}
L_R\left(\mathbf{x}, y_{\mathbf{x}} \right)= - \delta_{0/1}\left(\mathbf{c}\left(\mathbf{h}^1(\mathbf{x}),...,  \mathbf{h}^{N}(\mathbf{x})\right), y_{\mathbf{x}}\right)\log\left(1- \max_{i=1}^C\; \max_{j=1}^N  h^{j}_{i}(\mathbf{x}) \right).
	\label{eq_regularization}
\end{equation}
Here, $\delta_{0/1}(\mathbf{f}, y) \in \{0,1\}$ is an error function, where if the input classifier $\mathbf{f}$ can predict the correct label $y$,  it returns $0$, otherwise $1$.  
This design is also motivated by Assumption \ref{assu_simple_classification}, to encourage a weak base classifier to perform less poorly on challenging examples so that its chance of satisfying the acceptability assumption can be increased.

\subsection{Enhanced Training and Implementation} \label{implementation}

The proposed enhancement approach iGAT, supported by (1)   the  global adversarial examples generated and distributed following Section \ref{global_adv} and (2) the regularization term proposed in Section \ref{reg_misclass}, can be applied to any given ensemble adversarial defense method. 
We use $L_E$ to denote the original ensemble loss as in Eq. (\ref{eq_el}), the enhanced loss for training the base classifiers become
\begin{align}
	\label{eq_final_loss}
	\min_{\{\mathbf{h}^i\}_{i=1}^N} \;  \underbrace{\mathbb{E}_{(\mathbf{x}, y_{\mathbf{x}})\sim (\mathbf{X}, \mathbf y)} \left[L_{E}(\mathbf{x}, y_{\mathbf{x}})\right]}_{\textmd{original ensemble loss}} &\; +  \underbrace{\alpha \sum_{i=1}^N \mathbb{E}_{(\mathbf{x}, y_{\mathbf{x}})\sim \left(\tilde{\mathbf{X}}^i, \tilde{\mathbf{y}}^i\right)} \left[	 \ell_{CE}(\mathbf{h}^i( \mathbf{x}), y_{ \mathbf{x}})  \right]}_{\textmd{added global adversarial loss}}\\
	\nonumber
	&\; +  \underbrace{\beta \mathbb{E}_{(\mathbf{x}, y_{\mathbf{x}})\sim (\mathbf{X}, \mathbf y)\cup\left(\tilde{\mathbf{X}},  \tilde{\mathbf{y}} \right) } \left[L_R\left(\mathbf{x}, y_{\mathbf{x}} \right)\right]}_{\textmd{added misclassification regularization}},
\end{align}
where $\alpha, \beta\geq 0$ are hyper-parameters.
In practice, the base classifiers are firstly trained using an existing ensemble adversarial defense technique of interest, i.e., setting $\alpha=\beta=0$. If  some pre-trained base classifiers are available,  they can be directly used instead, and  fine-tuned with the complete loss. In our implementation, we employ the PGD attack to generate adversarial training examples, as it is the most commonly used in existing literature and in practice.

%% file: 5_experiments_and_results.tex
\section{Experiments and Results of iGAT}
\label{exp}

In the experiments, we compare with six state-of-the-art  ensemble adversarial defense techniques including ADP~\cite{pang2019improving}, CLDL~\cite{wang2023adversarial}, DVERGE~\cite{yang2020dverge},  SoE~\cite{cui2022synergy}, GAL~\cite{pang2019improving} and TRS~\cite{yang2021trs}.   The CIFAR-10  and CIFAR-100 datasets are used for evaluation, both containing 50,000 training and 10,000 test images~\cite{Zagoruyko2016WRN}. Overall,  ADP, CLDL, DVERGE and SoE appear to be the top  performing methods, and  we apply iGAT\footnote{The source codes and pre-trained models can be found at \href{https://github.com/xqsi/iGAT}{https://github.com/xqsi/iGAT}.} to enhance them.  
The enhanced, referred to as $\textmd{iGAT}_{\textmd{ADP}}$, $\textmd{iGAT}_{\textmd{CLDL}}$, $\textmd{iGAT}_{\textmd{DVERGE}}$ and $\textmd{iGAT}_{\textmd{SoE}}$, are compared with their original versions, and additionally GAL~\cite{pang2019improving} and TRS~\cite{yang2021trs}.

\subsection{Experiment Setting}

We test against  white-box attacks including PGD with $20$ inner optimization iterations and  CW with $L_{\infty}$ loss implemented by~\citet{wu2020adversarial}, and the black-box  SignHunter (SH) attack~\cite{al2020sign}  with $500$ maximum loss  queries. 
In accordance with~\citet{carmon2019unlabeled}, the CW attack is  applied on 1,000 equidistantly sampled testing examples.
We also test against the strongest AutoAttack (AA)~\cite{croce2020reliable}, which encapsulates variants of the PGD attack and the black-box  square  attack~\cite{andriushchenko2020square}.  
All attack methods use the perturbation strength $\varepsilon =8/255$.

For all the compared methods, an ensemble of $N=8$ base classifiers with ResNet-20~\cite{he2016deep} backbone is experimented, for which results of both the average and max output combiners are reported.
To implement the iGAT enhancement, the soft distributing rule from Eq. (\ref{eq_biased_sampling})  is used. The two hyper-parameters are set as $\alpha=0.25$ and $\beta=0.5$  for SoE, while $\alpha=5$ and $\beta=10$  for  ADP, CLDL and DVERGE,  found by grid search. Here SOE uses a   different parameter setting  because its loss construction differs from the others, thus it requires a different scale of the parameter range for tuning $\alpha$ and $\beta$. 
In practice, minor adjustments to hyper-parameters have little impact on the results.
The iGAT training uses a batch size of $512$,  and multi-step leaning rates of $\{0.01, 0.002\}$ for CIFAR10 and $\{0.1, 0.02, 0.004\}$ for CIFAR100. 
Implementation of existing methods uses either their pre-trained models or their source code for training that are publicly  available.
Each experimental run used one NVIDIA V100 GPU plus 8 CPU cores.

\begin{table}[!t]
	\caption{Comparison of classification accuracies in percentage reported on natural images and adversarial examples generated by different attack algorithms under  $L_\infty$-norm perturbation strength $\varepsilon=8/255$. The results are averaged over five independent runs. 
	The best performance is highlighted in bold,  the 2nd best underlined.}
	\label{tab_test_cifar}
	\vskip 0.15in
	\begin{center}
		\begin{small}
			\resizebox{0.95\textwidth}{!}{\begin{tabular}{l|r|ccccc|ccccc}
					\toprule
					\multicolumn{2}{c|}{} & \multicolumn{5}{c|}{Average Combiner (\%)} & \multicolumn{5}{c}{ Max Combiner (\%)}\\
					\cmidrule(r){3-12}  
					\multicolumn{2}{c|}{} & { Natural}  & { PGD} & { CW} & {SH} & {AA} & { Natural}  & { PGD} & { CW} & {SH} & {AA} \\
					\midrule
					\midrule	
					\multirow{5}{*}{\rotatebox[origin=c]{90}{CIFAR10}} 
					& TRS & $83.15$ & $12.32$ & $10.32$ & $39.21$ & $9.10$  & $82.67$ & $11.89$ & $10.78$ & $37.12$ & $7.66$ \\
					\cmidrule(r){2-12}
					& GAL  & $80.85$ &  $41.72$ & $41.20$ &  $54.94$ & $36.76$  & $80.65$ & $31.95$ & $27.80$ & $50.68$ & $9.26$  \\
					\cmidrule(r){2-12}
					& SoE & $82.19$ & $38.54$ & $37.59$ & $\textbf{59.69}$ & $32.68$   & $82.36$ & $32.51$ &$23.88$  & $41.04$ & $18.37$\\
					& $\textmd{iGAT}_{\textmd{SoE}}$  & $81.05$& ${40.58}$ & ${39.65}$ & $57.91$ & ${34.50}$   & $81.19$ & $31.98$ & $24.01$ & $40.67$& $19.65$\\
					\cmidrule(r){2-12}					
					& CLDL  & $84.15$ & $45.32$ & $41.81$ & $55.90$ & $37.04$  & $83.69$ & $39.34$ & $32.80$ & $51.63$ &  $15.30$ \\
					& $\textmd{iGAT}_{\textmd{CLDL}}$ & ${85.05}$ & $\underline{45.45}$ & ${42.00}$ & ${58.22}$ & ${37.14}$   & ${83.73}$ & ${40.84}$ & ${34.55}$ & ${51.70}$ & ${17.03}$\\
					\cmidrule(r){2-12}
					& DVERGE & $\underline{85.12}$ &  $41.39$ & $43.40$ & $57.33$ & $39.20$  & $\underline{84.89}$ & $\underline{41.13}$ & $\underline{39.70}$ & $\textbf{54.90}$ & $\underline{35.15}$\\
					& $\textmd{iGAT}_{\textmd{DVERGE}}$ & $\textbf{85.48}$ &  ${42.53}$ & $\underline{44.50}$ & ${57.77}$ & $\underline{39.48}$  & $\textbf{85.27}$ & $\textbf{42.04}$ & $\textbf{40.70}$ & $\underline{54.79}$ &  $\textbf{35.71}$ \\
					\cmidrule(r){2-12}
					& ADP  & $82.14$ &  $39.63$ & $38.90$ &  $52.93$ & $35.53$  & $80.08$ &  $36.62$ & $34.60$ &  $47.69$ & $27.72$ \\
					& $\textmd{iGAT}_{\textmd{ADP}}$  & ${84.96}$ &  $\textbf{46.27}$ & $\textbf{44.90}$ &  $\underline{58.90}$ & $\textbf{40.36}$  & ${80.72}$ & ${39.37}$ & ${35.00}$ & ${48.36} $& ${29.83}$ \\
					\midrule
					\midrule
					\multirow{5}{*}{\rotatebox[origin=c]{90}{CIFAR100}} 
					& TRS & $58.18$ & $10.32$ & $10.12$ & $15.78$ & $6.32$ & $57.21$ & $9.98$ & $9.23$ & $14.21$ & $4.34$  \\
					\cmidrule(r){2-12}
					& GAL  & $61.72$ &  $22.04$ & $\underline{21.60}$ &  $31.97$ & $\underline{18.01}$   & $59.39$ & $19.30$ & $13.60$ & $24.73$ & $10.36$ \\
					\cmidrule(r){2-12}
					& CLDL & $58.09$  & ${18.47}$   & ${18.01}$   & $29.33$ & ${15.52}$ & $55.51$ & $18.89$ & $13.07$ & $22.14$ & $4.51$\\
					& $\textmd{iGAT}_{\textmd{CLDL}}$ & ${59.63}$ & ${18.78}$ & ${18.20}$ & ${29.49}$ & $14.36$ & $56.91$ & $\underline{20.76}$ & $14.09$&  $20.43$ & $5.20$ \\
					\cmidrule(r){2-12}
					& SoE & $62.60$ & $20.54$ & $19.60$ & $\textbf{36.35}$ & $15.90$   & $\underline{62.62}$ & $16.00$ & $11.40$ & $24.25$ & $8.62$\\
					& $\textmd{iGAT}_{\textmd{SoE}}$ & $\textbf{63.19}$ & ${21.89}$ & ${19.70}$ & $\underline{35.60}$ & ${16.16}$   & $\textbf{63.02}$ & ${16.02}$ & ${11.45}$ & $23.77$ & ${8.95}$\\		
					\cmidrule(r){2-12}
					& ADP  & $60.46$ &  $20.97$ & $20.55$ &  $30.26$ & $17.37$   & $56.20$ & $17.86$ & $13.70$ & $21.40$ & $10.03$  \\
					& $\textmd{iGAT}_{\textmd{ADP}}$ & $60.17$ &  $\underline{22.23}$ & ${20.75}$ &  ${30.46}$ & ${17.88}$   & ${56.29}$ & ${17.89}$& ${14.10}$ & ${21.47}$ & ${10.09}$ \\
					\cmidrule(r){2-12}
					& DVERGE & $63.09$ &  $20.04$ & $20.01$ & $32.74$ & $17.27$  & $61.20$ & ${20.08}$ & $\underline{15.30}$ & $\underline{27.18}$ & $\underline{12.09}$ \\
					& $\textmd{iGAT}_{\textmd{DVERGE}}$ & $\underline{63.14}$ &  $\textbf{23.20}$ & $\textbf{22.50}$ & ${33.56}$ & $\textbf{18.59}$  & ${61.54}$ & $\textbf{20.38}$ & $\textbf{17.80}$ & $\textbf{27.88}$ & $\textbf{13.89}$ \\
					\bottomrule
				\end{tabular}
			}
		\end{small}
	\end{center}
	\vskip -0.1in
\end{table}

\begin{table}[t]
	\caption{Accuracy improvement in percentage by iGAT, i.e. $\frac{\textmd{iGAT- original}}{\textmd{original}}\times 100\%$,  reported on natural images and adversarial examples generated by different attack algorithms under  $L_\infty$-norm perturbation strength $\varepsilon=8/255$.}
	\label{tab_test_cifar1}
	\vskip 0.15in
	\begin{center}
		\begin{small}
			\resizebox{0.95\textwidth}{!}{\begin{tabular}{l|r|ccccc|ccccc}
					\toprule
					\multicolumn{2}{c|}{} & \multicolumn{5}{c|}{Average Combiner} & \multicolumn{5}{c}{ Max Combiner }\\
					\cmidrule(r){3-12}  
					\multicolumn{2}{c|}{} & { Natural}  & { PGD} & { CW} & {SH} & {AA} & { Natural}  & { PGD} & { CW} & {SH} & {AA} \\
					\midrule
					\midrule
					\multirow{3}{*}{\rotatebox[origin=c]{90}{CIFAR10}} 
					& ADP   &  $+{3.43}$ &  $+ 16.75$ & $+{15.42}$ &  $+{11.28}$ & $+{13.59}$  & $+{0.80}$ & $+{7.51}$ & $+{1.16}$ & $+{1.40}$ & $+{7.61}$ \\
					\cmidrule(r){2-12}
					& DVERGE  & $+{0.42}$ &  $+{2.75}$ & $+{2.53}$ & $+{0.77}$ & $+{0.71}$  & $+{0.45}$ & $+{2.21}$ & $+{2.52}$ & $-0.20$ &  $+{1.59}$\\
					\cmidrule(r){2-12}					
					& SoE  & $-1.39$ & $+5.29$ & $+5.48$ & $-2.98$ & $+5.57$ & $-1.42$ & $-1.63$ & $+0.54$ & $-0.90$ & $+6.97$\\
					\cmidrule(r){2-12}					
					&  CLDL  & $+{1.07}$ & $+{0.29}$ & $+{0.45}$ & $+{4.15}$ & $+{0.27}$ & $+{0.05}$ & $+{3.81}$ & $+{5.34}$ & $+{0.14}$ & $+{11.31}$\\
					\midrule
					\midrule
					\multirow{3}{*}{\rotatebox[origin=c]{90}{CIFAR100}} 
					& ADP  & $-0.48$ &  $+{6.01}$ & $+{0.97}$ & $+{0.66}$ & $+{2.94}$   & $+{0.16}$ & $+{0.17}$ & $+{2.92}$ & $+{0.33}$ & $+{0.60}$  \\
					\cmidrule(r){2-12}
					&  DVERGE  & $+{0.08}$ &  $+{15.77}$ & $+{12.44}$ & $+{2.50}$ & $+{7.64}$  & $+{0.56}$ & $+{1.49}$ & $+ 16.34$ & $+{2.58}$ & $+{14.89}$ \\
					\cmidrule(r){2-12}
					& SoE & $+{0.94}$ & $+{6.57}$ & $+{0.51}$ & $-2.06$ & $+{1.64}$ & $+{0.64}$ & $+{0.13}$ & $+{0.44}$ & $-1.98$ & $+{3.83}$\\	
					\cmidrule(r){2-12}					
					&  CLDL & $+2.65$ & $+1.68$ & $+1.05$ & $+0.55$ & $-7.47$ & $+2.52$ & $+9.90$ & $+7.80$ & $-7.72$ & $+15.30$\\	
					\bottomrule
				\end{tabular}
			}
		\end{small}
	\end{center}
	\vskip -0.1in
\end{table}

\subsection{Result Comparison and Analysis}

We  compare different defense approaches by reporting their classification accuracies computed  using  natural images and   adversarial examples generated by different attack algorithms, and report the results in Table \ref{tab_test_cifar}. 
The proposed enhancement has lifted the performance of  ADP and DVERGE to a state-of-the-art level for CIFAR-10 under most of the examined attacks, including both the white-box and black-box ones. 
The enhanced DVERGE by iGAT has outperformed all the compared methods in most cases for CIFAR-100. 
In addition, we report in Table \ref{tab_test_cifar1} the accuracy improvement  obtained by iGAT for the studied  ensemble defense algorithms, computed as their accuracy difference normalised by the accuracy of the original algorithm.
It can be seen that iGAT has positively improved the baseline methods in almost all cases. 
In many cases, it has achieved an accuracy boost over $10\%$.

Here are some further discussions on the performance. We observe that  DVERGE and its iGAT enhancement  perform proficiently on both CIFAR10 and CIFAR100, while ADP  and its enhancement are less robust on CIFAR100. 
We attempt to explain this by delving into the algorithm nature of DVERGE and ADP.  
The ADP design  encourages prediction disparities among the base models.  
As a result, each base model becomes proficient in classifying a subset of classes that the other base models may struggle with.
However, a side effect of this  is to discourage base models from becoming good at overlapping classes, which may become ineffective when having to handle a larger number of classes.
The reasonably good improvement achieved by iGAT for ADP in  Table \ref{tab_test_cifar1} indicates that an addition of global adversarial examples is able to rescue such situation to a certain extent.
On the other hand,   in addition to encouraging  adversarial diversity among the base models, DVERGE also aims at a stable classification so that each example is learned by multiple base models. 
This potentially makes it suitable for handling both large and small numbers of classes.
Moreover, we also observe that the average combiner provides better performance than the max combiner in general. 
The reason can be  that an aggregated prediction from multiple well-trained base classifiers is more statistically stable.

 \begin{table}[t]
	\caption{Results of ablation studies based on  $\textmd{iGAT}_{\textmd{ADP}}$ using   CIFAR-10 under the PGD attack. The results are averaged over five independent runs. The best performance is highlighted in bold.}
	\label{tab_ablation}
	\begin{center}
			 \resizebox{0.95\textwidth}{!}{
			 \begin{tabular}{l|ccccc}
					\toprule
					&  Opposite Distributing & Random Distributing & Hard Distributing& $\beta=0$ & $\textmd{iGAT}_{\textmd{ADP}}$ \\
					\midrule 
					\midrule
					Natural (\%) & $82.45$ & $83.05$ & $83.51$ & $83.45$ & $\textbf{84.96}$ \\
					PGD (\%)   & $41.31$ & $42.60$ & $44.21$ & $42.32$ & $\textbf{46.25}$  \\
					\bottomrule
			\end{tabular}
			}
	\end{center}
	\vskip -0.1in
\end{table}

\subsection{  Ablation Studies } \label{ablation}

The key designs of iGAT include its distributing rule and the regularization term. 
We perform ablation studies to examine their effectiveness. Firstly, we compare the used soft distributing rule with three alternative distributing rules, including   
(1)  a distributing rule opposite to the proposed, which allocates the adversarial examples to the base models that produce the lowest prediction score, 
(2) a random distributing rule by replacing Eq. (\ref{eq_biased_sampling}) by a uniform distribution, 
and (3) the hard distributing rule  in Eq. (\ref{eq_split_global_data}).
Then, we  compare with the setting of $\beta =0$ while keeping the others unchanged. 
This change removes the proposed regularization term.  
Results are reported in Table \ref{tab_ablation} using  $\textmd{iGAT}_{\textmd{ADP}}$ with the average combiner, evaluated by CIFAR-10  under the PGD attack.
It can be seen that a change or removal of a key design  results in obvious performance drop, which verifies the effectiveness of the design.

%% file: supp_UsedExisting.tex
\section{Studied Ensemble Adversarial Defense Techniques } \label{existing}

We briefly explain four ensemble adversarial defense techniques including ADP~\cite{pang2019improving}, CLDL~\cite{wang2023adversarial}, DVERGE~\cite{yang2020dverge} and SoE~\cite{cui2022synergy}.
They are used to test the  proposed enhancement approach iGAT. 
In general, an ensemble  model contains multiple base models, and the training is conducted by minimizing their classification losses together with a diversity measure.  
The output of each base model contains the probabilities of an example belonging to  the $C$ classes. For an input example $\mathbf{x}\in \mathcal{X}$, we denote its output from the $i$-th base  classifier by $\mathbf{h}^i(\mathbf{x})=\left[h_1^i(\mathbf{x})...,h_C^i(\mathbf{x})\right]$ for $i\in[N]$, where $N$  denotes the  base model number. 

\subsection{ADP Defense}
ADP employs an ensemble by averaging, i.e., $\mathbf{h}(\mathbf{x}):= \frac{1}{N}\sum_{i=1}^N\mathbf{h}^i(\mathbf{x})$. 
The  base classifiers are trained  by minimizing a  loss that combines (1) the cross entropy loss of each base classifier, (2) the Shannon entropy of the ensemble prediction for regularization, and (3) a diversity measure to encourage different predictions by the base classifiers.  
Its formulation is exemplified below using one training  example $(\mathbf{x}, y_{\mathbf{x}})$:
\begin{equation}
	\label{eq_adp}
	L_{\textmd{ADP}}(\mathbf{x}, y_{\mathbf{x}})= \underbrace{\sum_{i=1}^N \ell_{CE}(\mathbf{h}^i(\mathbf{x}), y_{\mathbf{x}})}_{\textmd{classification loss}} \underbrace{-  \alpha H\left(\mathbf{h}(\mathbf{x})\right)}_{\substack{\textmd{uncertainty} \\ \textmd{regularization}}} +\underbrace{\beta\log(D(\mathbf{h}^1(\mathbf{x}), \mathbf{h}^2(\mathbf{x}), \ldots,\mathbf{h}^N(\mathbf{x}), y_{\mathbf{x}}))}_{\textmd{prediction diversity}} ,
\end{equation}
where $\alpha, \beta \geq 0$ are hyperparameters, the Shannon entropy  is  $H(\mathbf{p})=-\sum_{i=1}^C p_i \log(p_i)$,  and $D(\mathbf{h}^1, \mathbf{h}^2, \ldots \mathbf{h}^N, y)$ measures the geometric diversity between $N$ different $C$-dimensional  probability vectors.   
To compute the diversity, a normalized $(C-1)$-dimensional vector $\tilde{\mathbf{h}}^i_{\backslash y}$ is firstly obtained by removing from $\mathbf{h}^i$ the element  at the position $y\in[C]$,   the resulting   $\left\{\tilde{\mathbf{h}}^i_{\backslash y}\right\}_{i=1}^N$ are stored as columns of the  $(C-1)\times N $ matrix $\tilde{\mathbf{H}}_{\backslash y}$, and then it has   $D(\mathbf{h}^1, \mathbf{h}^2, \ldots \mathbf{h}^N, y) =\det\left(\tilde{\mathbf{H}}^T_{\backslash y}\tilde{\mathbf{H}}_{\backslash y}\right)$.

\subsection{CLDL  Defense}
CLDL   provides an alternative way to formulate the diversity between base classifiers, considering both the base classifier prediction and its loss gradient. Its loss for the training example $(\mathbf{x}, y_{\mathbf{x}})$  is given by 
\begin{align}
	\label{eq_cldl}
	\nonumber
	L_{\textmd{CLDL}}(\mathbf{x}, y_{\mathbf{x}})  =& \underbrace{\frac{1}{N}  \sum_{i=1}^N D_{\textmd{KL}} \left(\mathbf{s}^i(\mathbf{x})||\mathbf{h}^i(\mathbf{x})\right)  }_{\textmd{classification loss}} \\
	\nonumber
	&\underbrace{-\alpha \log\left(\frac{2}{N(N-1)}\sum_{ i=1}^N\sum_{j=i+1}^N e^{\textmd{JSD}\left(\mathbf{s}^i_{\backslash}(\mathbf{x})|| \mathbf{s}^j_{\backslash}(\mathbf{x}) \right) } \right)}_{\textmd{prediction diversity}}\\
	&+ \underbrace{\frac{2\beta}{N(N-1)}\sum_{i=1}^N\sum_{j=i+1}^N \cos(\nabla  D_{\textmd{KL}} (\mathbf{s}^i(\mathbf{x})|| \mathbf{h}^i(\mathbf{x})), \nabla  D_{\textmd{KL}}\left (\mathbf{s}^j(\mathbf{x} )|| \mathbf{h}^j (\mathbf{x}) \right)}_{\textmd{gradient diversity}},
\end{align}
where $\mathbf{s}^i(\mathbf{x})$ is a soft label vector computed for $(\mathbf{x}, y_{\mathbf{x}})$  
by a label  smoothing technique called label confusion model~\citep{guo2021label}.
The vector $\mathbf{s}^i_{\backslash}$ is defined as a $(C-1)$-dimensional vector by removing from $\mathbf{s}^i$ its maximal value. 
The  Kullback–Leibler (KL) divergence is used to examine the difference between the soft label vector  and the prediction vector, serving as a soft  version of  the classification loss. 
The other used divergence measure is Jensen–Shannon divergence  (JSD), given as $\textmd{JSD}\left(\mathbf{p} || \mathbf{q}\right) =\frac{1}{2}\left(\textmd{D}_{\textmd{KL}}\left(\mathbf{p} || \mathbf{g}\right) +\textmd{D}_{\textmd{KL}}\left(\mathbf{q} || \mathbf{g}\right) \right)$ with $ \mathbf{g} =\frac{1}{2}(\mathbf{p} + \mathbf{q})$.

\subsection{DVERGE Defense}
DVERGE proposes a vulnerability diversity  to help training the base classifiers with improved adversarial robustness. For training the $i$-th base classifier, it minimizes  
\begin{equation}
	\label{eq_dverge0}
	L^{\textmd{original}}_{\textmd{DVERGE}}(\mathbf{x}, y_{\mathbf{x}})  = \underbrace{ \ell_{CE}(\mathbf{h}^i(\mathbf{x}), y_{\mathbf{x}})}_{\textmd{classification loss}}  +   \underbrace{\alpha \sum_{ j\neq i} \mathbb{E}_{(\mathbf{x}_s, y_{\mathbf{x}_s})\sim D, l\in [L]} \left[\ell_{\textmd{CE}}\left(\mathbf{h}^i \left( \tilde{\mathbf{x}}\left({\mathbf{h}^j_{(l)}}, \mathbf{x}, \mathbf{x}_s\right)\right), y_{\mathbf{x}_s}\right)\right]}_{\textmd{adversarial vulnerability diversity}},
\end{equation}
where $\alpha\geq 0$ is a hyperparameter. 
Given an input example $\mathbf{x}$, $\tilde{\mathbf{x}}\left({\mathbf{h}^j_{(l)}}, \mathbf{x}, \mathbf{x}_s\right)$ computes  its distilled non-robust feature vector proposed by ~\citet{ilyas2019adversarial}. 
This   non-robust feature vector is computed  with respect to the  $l$-th layer of the $j$-th base classifier with its mapping function denoted by $\mathbf{h}^j_{(l)}$  and a randomly sampled natural example $ \mathbf{x}_s$, by
\begin{align}
	\tilde{\mathbf{x}}\left({\mathbf{h}^j_{(l)}}, \mathbf{x}, \mathbf{x}_s\right) = \arg\min_{\mathbf{z}\in \mathbb{R}^d} &\; \left\|\mathbf{h}^j_{(l)}(\mathbf{z}) - \mathbf{h}^j_{(l)}(\mathbf{x})\right\|_2^2, \\
	\nonumber
	\textmd{s.t. } & \;\left\|\mathbf{z} - \mathbf{x}_s\right\|_\infty\le \epsilon.
\end{align}
When $\mathbf{x}$ and $\mathbf{x}_s$ belong to different classes, $\tilde{\mathbf{x}}$ can be viewed as an adversarial example that is visually  similar to $\mathbf{x}_s$ but is classified by the $j$-th base classifier into the same class as $\mathbf{x}$. 
This represents a weakness of $\mathbf{h}^j$, and as a correction, the $i$-th base classifier is trained to correctly classify $\tilde{\mathbf{x}}$ into the same class as $\mathbf{x}_s$. 
But when $\mathbf{x}$ and $\mathbf{x}_s$ come from the same class, $(\tilde{\mathbf{x}}, y_{\mathbf{x}_s})$ is just an example similar to the natural one $(\mathbf{x}_s, y_{\mathbf{x}_s}) \in D$, for which the first and second loss terms play similar roles.  
Therefore, DVERGE simplifies the above loss  in practice, and trains each base classifier by
\begin{equation}
	\label{eq_dverge}
	\min_{\mathbf{h}^i} L_{\textmd{DVERGE}}(\mathbf{x}, y_{\mathbf{x}})  =  \mathbb{E}_{(\mathbf{x}_s, y_{\mathbf{x}_s})\sim D, l\in [L]} \left[ \sum_{ j\neq i} \ell_{\textmd{CE}}\left(\mathbf{h}^i \left( \tilde{\mathbf{x}}\left({\mathbf{h}^j_{(l)}}, \mathbf{x}, \mathbf{x}_s\right)\right), y_{\mathbf{x}_s}\right)\right].
\end{equation}
It removes the  classification loss on the natural data. 

\subsection{SoE  Defense}
SoE proposes a version of  classification loss using adversarial examples and  a surrogate loss that acts similarly to the vulnerability diversity loss as in DVERGE. 
For each base classifier $\mathbf{h}^i$, an auxiliary scalar output head $g^i$ is used to approximate  its predicted probability  for the true class. 
 Its overall loss exemplified by the training example $(\mathbf{x}, y_{\mathbf{x}})$ is given as
\begin{equation}
	\label{eq_soe}
	L_{\textmd{SoE}}(\mathbf{x}, y_{\mathbf{x}})  = \underbrace{\sum_{j=1}^N \ell_{BCE}\left(h_{y_\mathbf{x}}^j\left(\tilde{\mathbf{x}}^i\right), g^j\left(\tilde{\mathbf{x}}^i\right) \right)}_{\textmd{adversarial classification loss}}  -   \underbrace{\sigma\ln \sum_{j=1}^N\exp\left( \frac{-\ell_{CE}\left(\mathbf{h}^j\left(\tilde{\mathbf{x}}^i\right), y_{\mathbf{x}}\right)}{\sigma}\right) }_{\textmd{surrogate loss for vulnerability diversity}},
\end{equation}
where $\ell_{BCE}$ is the binary cross entropy loss, and $\sigma>0$ is the weight parameter.
Adversarial examples are generated to compute the losses by using the PGD attack. 
For the $j$-th base classifier, the attack is applied to each $i$-th ($i\ne j$) base classifer to generate training data, resulting in $\tilde{\mathbf{x}}^i=\phi(\mathbf{h}^i, \mathbf{x}, \textmd{PGD})$.
SoE has two training phases and in the second training phase, rather than using $\tilde{\mathbf{x}}^i$, a different adversarial example is generated by $\tilde{\mathbf{x}}=\phi(\mathbf{h}^{k}, \mathbf{x}, \textmd{PGD})$ where $k=\arg \max_{i\in[N]} g^i(\mathbf{x})$, aiming at attacking the  best-performing base classifier.

%% file: supp_lemma.tex
\section{Proof of Theoretical Results}

Given a $C$-class $L$-layer MLP $\mathbf{h}:\mathcal{X}\to [0,1]^C$ described in Assumption \ref{MLP_assu}, we study its cross-entropy loss for one example $(\mathbf{x}, y_{\mathbf{x}})$, i.e.,  $  \ell_{CE}(\mathbf{h}(\mathbf{x}), y_{\mathbf{x}}) = -\log h_{y_{\mathbf{x}}}(\mathbf{x})$, where  its partial derivative with respect to the $k$-th element of $\mathbf{x}$ is given by
\begin{equation}
\frac{\partial \ell_{CE}(\mathbf{x})}{\partial x_k} =\sum_{i=1}^C\left(h_i(\mathbf{x})-\Delta_{i, y_\mathbf{x}}\right)\frac{\partial z_i }{\partial x_k},
\label{eq_first_order}
\end{equation}
where $\Delta_{i, y_\mathbf{x}}= 
\left\{
\begin{array}{cc }
  1, &   \textmd{if } i=y_\mathbf{x},   \\
  0, &      \textmd{otherwise.}
\end{array}
\right.$
Perturbing the  input $\mathbf{x}$ to $\mathbf{x}+\bm\delta$, sometimes we simplify the notation of the perturbed function output, for instance, $\tilde{\ell}(\mathbf{x}) = \ell(\mathbf{x}+\bm{\delta})$,  $\tilde{\mathbf{h}}(\mathbf{x}) = \mathbf{h}(\mathbf{x}+\bm{\delta})$, $ \tilde{\mathbf{z}}(\mathbf{x}) = \mathbf{z}(\mathbf{x}+\bm{\delta})$ and $ \tilde{\sigma}(\mathbf{x}) = \sigma(\mathbf{x}+\bm{\delta})$.

Our main theorem builds on  a supporting Lemma \ref{lem_2}. 
In the lemma, we derive an upper bound for the difference between the  predictions $\mathbf{h}(\mathbf{x})$ and  $\mathbf{h}(\mathbf{z})$ for two examples, computed by  an MLP $\mathbf{h}: \mathbb{R}^d\rightarrow [0,1]^C$  satisfying Assumption \ref{MLP_assu}.
Before proceeding to prove the main theorem, we provide a proof sketch.
For each ambiguous pair, we firstly analyse its 0/1 risk under different situations when being classified by a single classifier, and derive its empirical 0/1 risk as $r_1 = 1-p+\frac{1}{2}p^2$. Then we analyse the 0/1 risk for this pair under different situations when being classified by an ensemble classifier, where both $\max$ and average combiners are considered.  We derive the  ensemble empirical 0/1 risk as $r_2= 1-3p^2+3p^3-\frac{3}{4}p^4$. Finally, we prove the main result in Eq. (\ref{lemma_result}) by obtaining a sufficient condition for achieving a reduced ensemble risk, i.e., $p> 0.425$ which enables $r_2\leq r_1$.

\subsection{Lemma \ref{lem_2} and Its Proof}

\begin{lem} 
Suppose a  $C$-class $L$-layer MLP $\mathbf{h}: \mathbb{R}^d\rightarrow [0,1]^C$ with softmax  prediction layer satisfies  Assumption \ref{MLP_assu}. For any $\mathbf{x}, \mathbf{z} \in \mathbb{R}^d$ and $c=1,2\ldots, C$ , the following holds 
\begin{equation}
\label{bound2}
	 \left|h_{c}(\mathbf{x}) - h_{c}(\mathbf{z} )\right| \le \|\mathbf{x}-\mathbf{z} \|_2 B\sqrt{C\left(\tilde{\lambda}^2 - \xi\right)}
	\end{equation} 
for some constant $\xi\le \tilde{\lambda}^2$, where   $\tilde{\lambda}$ and $B$ are constants associated with the MLP family under  Assumption \ref{MLP_assu}.
	\label{lem_2}
\end{lem}

\begin{proof}
Define the perturbation vector $\bm\delta \in \mathbb{R}^d$ such that $\mathbf{z} = \mathbf{x} + \bm\delta$ and denote its strength by $\epsilon = \|\bm\delta\|_2$, these will be used across the proof.
We start from the cross-entropy loss curvature measured by   Eq. (\ref{app_cur}), given as
\begin{equation}
\label{eq_support1}
\lambda_\mathbf{h}^2(\mathbf{x},\bm\delta)  =  \frac{1}{\epsilon^2}\left\| \nabla  \ell_{CE}(\mathbf{h}(\mathbf{x}), y_{\mathbf{x}}) -   \ell_{CE}(\mathbf{h}(\mathbf{x} +\bm\delta), y_{\mathbf{x}})  \right\|_2^2  = \frac{1}{\epsilon^2}\sum_k \left (\frac{\partial \ell_{CE}(\mathbf{x})}{\partial x_k} - \frac{\partial \tilde{\ell}_{CE}(\mathbf{x})}{\partial x_k}\right)^2.
\end{equation}
 Below we will expand this curvature expression, where we denote a perturbed function $f(\mathbf{x})$ by using $\tilde{f}(\mathbf{x}) $ and $f(\mathbf{x}+ \bm\delta )$ interchangeably.

By Eq. (\ref{eq_first_order}), it has 
\begin{equation}
\label{eq_support2}
 \left|\frac{\partial \ell_{CE}(\mathbf{x})}{\partial x_k} - \frac{\partial \tilde{\ell}_{CE}(\mathbf{x})}{\partial x_k}\right| = \left|\sum_{c=1}^C\left(h_c(\mathbf{x})-\Delta_{i, y_\mathbf{x}}\right)\frac{\partial z_i }{\partial x_k}-  \sum_{c=1}^C\left(\tilde{h}_c(\mathbf{x})-\Delta_{i, y_\mathbf{x}}\right)\frac{\partial \tilde{z}_i }{\partial x_k} \right|. \\
\end{equation} 
Working with the MLP formulation, it is straightforward to express the quantity $\frac{\partial z_i }{\partial x_k}$ in terms of the derivatives of the activation functions and the neural network weights, as
\begin{equation}
\label{eq_support3}
\frac{\partial z_i }{\partial x_k}  =\frac{\partial \sum_{s_L} w^{(L)}_{i,s_L}a^{(L-1)}_{s_L} (\mathbf{x}) }{\partial x_k}=\sum_{s_L} w^{(L)}_{i,s_L}\frac{\partial a^{(L-1)}_{s_L} (\mathbf{x}) }{\partial x_k}.
\end{equation}
For the convenience of explanation, we simplify the notation by defining $g_{(L-1),s_L} (\mathbf{x}) = \frac{\partial a^{(L-1)}_{s_L}(\mathbf{x})}{\partial x_k}$, and we have 
\begin{align}
\label{eq_support4}
\frac{\partial z_i }{\partial x_k}  & = \sum_{s_L} w^{(L)}_{i,s_L} g_{(L-1),s_L}(\mathbf{x}) , \\
\frac{\partial \tilde{z}_i }{\partial x_k}  & =\sum_{s_L} w^{(L)}_{i,s_L} \tilde{g}_{(L-1),s_L}(\mathbf{x}).
\end{align}
Applying multivariate Taylor expansion \cite{folland2005higher}, we obtain 
\begin{equation}
\tilde{g}_{(L-1),s_L}(\mathbf{x}) = g_{(L-1),s_L} (\mathbf{x}) + \sum_{k=1}^{d}\frac{\partial g_{(L-1),s_L}(\mathbf{x})}{\partial x_k} \delta_k + \sum_{n\geq 2} \left(\sum_{\substack{a_k\in \mathcal{Z}_0, \\  k\in[d], \\\sum_{k=1}^d a_k = n}}C_n^{(a_1,...,a_d)}\delta_1^{a_1}...\delta_d^{a_d}\right),
\end{equation}
where $\mathcal{Z}_0$ denotes the set of nonnegative integers, $\delta_k$ is the $k$-th element of the perturbation vector $\bm\delta$,  and $C_{n}^{(a_1,...,a_d)}$ denotes the coefficient of each higher-order term of $\delta_1^{a_1}...\delta_d^{a_d}$.
Combining the above equations, we have
\begin{align}
\label{eq_support5}
&\sum_k \left(\frac{\partial \ell_{CE}(\mathbf{x})}{\partial x_k} - \frac{\partial \tilde{\ell}_{CE}(\mathbf{x})}{\partial x_k}\right)^2 \\
\nonumber
= & \sum_k \left(\sum_{i=1}^C\sum_{s_L} \left(h_i(\mathbf{x})g_{(L-1),s_L}(\mathbf{x}) - \tilde{h}_i(\mathbf{x})\tilde{g}_{(L-1),s_L}(\mathbf{x})\right)w^{(L)}_{i,s_L}   - \right.\\
\nonumber
& \left. \sum_{s_L}\left( g_{(L-1),s_L} (\mathbf{x}) -  \tilde{g}_{(L-1),s_L}(\mathbf{x}) \right)w^{(L)}_{y_{\mathbf{x}},s_L}   \right)^2 \\
\nonumber
=& \underbrace{\sum_{k }\left( \sum_{i=1}^C\sum_{s_L} \left(h_i(\mathbf{x})- \tilde{h}_i(\mathbf{x})\right) g_{(L-1),s_L}(\mathbf{x})  w^{(L)}_{i,s_L} \right)^2}_{T(\mathbf{x})} + \underbrace{\sum_{n\geq 1}\left(\sum_{\substack{a_k\in \mathcal{Z}_0, \\  k\in[d], \\\sum_{k=1}^d a_k = n}}D_n^{(a_1,...,a_d)}\delta_1^{a_1}...\delta_d^{a_d}\right)}_{S(\mathbf{x})},
\end{align} 
where  $D_n^{(a_1,...,a_d)}$ denotes the coefficient of $\delta_1^{a_1}...\delta_d^{a_d}$, computed from the terms like $h_i(\mathbf{x})$, $\tilde{h}_i(\mathbf{x})$, $C_{n}^{(a_1,...,a_d)}$ and the neural network weights.
Define  a $C$-dimensional column vector $\mathbf{p}^{(k)}$ with its $i$-th element computed by  $p^{(k)}
_i = \sum_{s_L} g_{(L-1),s_L}(\mathbf{x}) w^{(L)}_{i,s_L} $ and a matrix $\mathbf{P}_\mathbf{h} = \sum_{k}\mathbf{p}^{(k)}{\mathbf{p}^{(k)}}^T$, the term $T(\mathbf{x})  $ can be rewritten as
\begin{align}
	\label{eq_tx_definition}
	T(\mathbf{x})  &= \sum_{k}\left(\left(\mathbf{h}(\mathbf{x})- \tilde{\mathbf{h}}(\mathbf{x})\right)^T\mathbf{p}_k\right)^2 =\left(\mathbf{h}(\mathbf{x})- \tilde{\mathbf{h}}(\mathbf{x})\right)^T \mathbf{P}_{\mathbf{h}} \left(\mathbf{h}(\mathbf{x})- \tilde{\mathbf{h}}(\mathbf{x})\right).
\end{align}
The   factorization $\mathbf{P}_{\mathbf{h}} = \mathbf{M}_\mathbf{h}\mathbf{M}_\mathbf{h}^T$ can be obtained by conducting singular value decomposition of $\mathbf{P}_{\mathbf{h}} $. 
The above new expression of $T(\mathbf{x}) $ helps  bound   the difference between $\mathbf{h}(\mathbf{x})$ and  $\tilde{\mathbf{h}}(\mathbf{x})$.

According to the norm definition, we have 
\begin{align}
	\|\mathbf{M}_\mathbf{h}\|_2 & =\max_{\mathbf{q}\in \mathbb{R}^{d}\ne \mathbf{0}} \frac{\|\mathbf{M}_\mathbf{h}\mathbf{q}\|_2}{\|\mathbf{q}\|_2}   =\max_{\mathbf{q}\in\mathbb{R}^{C}\ne \mathbf{0}} \frac{\|\mathbf{q}^T\mathbf{M}_\mathbf{h}\|_2}{\|\mathbf{q}\|_2}, \\
	\|\mathbf{M}_\mathbf{h}^\dagger\|_2 & =\max_{\mathbf{q}\in \mathbb{R}^{C}\ne \mathbf{0}} \frac{\|\mathbf{M}_\mathbf{h}^\dagger\mathbf{q}\|_2}{\|\mathbf{q}\|_2}   =\max_{\mathbf{q}\in\mathbb{R}^{d}\ne \mathbf{0}} \frac{\|\mathbf{q}^T\mathbf{M}_\mathbf{h}^\dagger\|_2}{\|\mathbf{q}\|_2}.
\end{align}

Subsequently, the following holds for any nonzero $\mathbf{q}\in \mathbb{R}^{C}$ and  $\mathbf{p}\in \mathbb{R}^{d}$
\begin{align}
&\|\mathbf{q}^T\mathbf{M}_\mathbf{h}\|_2\leq \|\mathbf{M}_\mathbf{h}\|_2\|\mathbf{q}\|_2, \\
&\|\mathbf{p}^T\mathbf{M}_\mathbf{h}^\dagger\|_2\leq \|\mathbf{M}_\mathbf{h}^\dagger\|_2\|\mathbf{p}\|_2.
\end{align}
Letting $\mathbf{q}=\mathbf{h}(\mathbf{x})- \tilde{\mathbf{h}}(\mathbf{x})$ and using the fact that each element in $\mathbf{h}(\mathbf{x})$ and $\tilde{\mathbf{h}}(\mathbf{x})$ is a probability value less than 1, it has 
\begin{align}
	T(\mathbf{x}) = &\; \left\|\left(\mathbf{h}(\mathbf{x})- \tilde{\mathbf{h}}(\mathbf{x})\right)^T\mathbf{M}_\mathbf{h}\right\|_2^2 \le \; \left\|\mathbf{M}_\mathbf{h}\right\|_2^2\left\|\mathbf{h}(\mathbf{x}) - \tilde{\mathbf{h}}(\mathbf{x})\right\|_2^2\le \left( \sup_{\mathbf{h}}  \left\|\mathbf{M}_\mathbf{h}\right\|_2 \right)^2 C,
\end{align}
which results in the fact that $T(\mathbf{x})$ is upper bounded by Assumption \ref{MLP_assu} where  $\left\|\mathbf{M}_{\mathbf{h}}\right\|_2 \leq B_0$. 
Letting $\mathbf{p}=\mathbf{M}_\mathbf{h}^T\left(\mathbf{h}(\mathbf{x})- \tilde{\mathbf{h}}(\mathbf{x})\right)$ and using the Assumption \ref{MLP_assu} where   $\left\|\mathbf{M}_\mathbf{h}^\dagger\right\|_2 \leq B $,   it has 
\begin{align}
\label{eq_non_cur1}
\nonumber
\|\mathbf{h}(\mathbf{x})- \tilde{\mathbf{h}}(\mathbf{x})\|_2& =   \left\|\left(\mathbf{h}(\mathbf{x})- \tilde{\mathbf{h}}(\mathbf{x})\right)^T\mathbf{M}_\mathbf{h}\mathbf{M}_\mathbf{h}^\dagger\right\|_2 \\
&  \le  \left\|\mathbf{M}_\mathbf{h}^\dagger\right\|_2\left\|\left(\mathbf{h}(\mathbf{x})- \tilde{\mathbf{h}}(\mathbf{x})\right)^T\mathbf{M}_\mathbf{h}\right\|_2  \leq B\sqrt{T(\mathbf{x})}.
\end{align}

Now we focus on analyzing $T(\mathbf{x})$. 
Working with Eq. (\ref{eq_support5}) and considering the fact that $\sum_k \left(\frac{\partial \ell_{CE}(\mathbf{x})}{\partial x_k} - \frac{\partial \tilde{\ell}_{CE}(\mathbf{x})}{\partial x_k}\right)^2$ is a positive term and $T(\mathbf{x})$ is upper bounded, $S(\mathbf{x})$  has to be lower bounded.  We express this lower bound by  $\xi\epsilon^2$ using a constant $\xi$ for the convenience of later derivation, resulting in 
\begin{equation}
\label{sx_bound}
S(\mathbf{x}) \geq \xi\epsilon^2.
\end{equation}
Given the perturbation strength $\epsilon^2 = \|\bm\delta\|_2^2$, applying the curvature assumption in Assumption \ref{MLP_assu}, i.e., $\lambda_{\mathbf{h}}(\mathbf{x}, \bm\delta)\leq \tilde{\lambda}$, also Eqs. (\ref{eq_support1}), (\ref{eq_support5}) and (\ref{sx_bound}), it has
\begin{equation}
	\label{eq_tx_ineq}
T(\mathbf{x})  + \xi\epsilon^2 \leq \tilde{\lambda}^2\epsilon^2 \Rightarrow  T(\mathbf{x})  \leq      (\tilde{\lambda}^2 -\xi)\epsilon^2. 
\end{equation}
Incorporating  this into Eq. (\ref{eq_non_cur1}), it has 
\begin{equation}
\label{eq_non_cur}
\|\mathbf{h}(\mathbf{x})- \tilde{\mathbf{h}}(\mathbf{x})\|_2  \le \epsilon B\sqrt{\tilde{\lambda}^2 - \xi}.
\end{equation}

Applying the inequality of $\sum_{i=1}^m a_i^2 \geq \frac{1}{m} \left(\sum_{i=1}^m a_i\right)^2 $, also the fact   $\sum_{c=1}^C h_c(\mathbf{x}) =\sum_{c=1}^C \tilde{h}_c(\mathbf{x}) =1$,  the following holds for any class $c \in \{1,2,\ldots, C\}$:
\begin{align}
	\nonumber
	\|\mathbf{h}(\mathbf{x})- \tilde{\mathbf{h}}(\mathbf{x})\|_2^2  \ge &  \sum_{j\ne c} \left|h_j(\mathbf{x})-\tilde{h}_j(\mathbf{x})\right|^2 \ge\frac{1}{C-1}\left(\sum_{j\ne c} \left|h_j(\mathbf{x})-\tilde{h}_j(\mathbf{x}) \right|\right)^2 \\
	\ge &  \frac{1}{C} \left|\sum_{j\ne   c} \left(h_j(\mathbf{x})-\tilde{h}_j(\mathbf{x}) \right)\right| ^2   = \frac{1}{C} \left| h_{c}(\mathbf{x})-\tilde{h}_{c}(\mathbf{x}) \right|^2 .
	\label{eq_non_del}
\end{align} 
Incorporating Eq. (\ref{eq_non_cur}) to the above, we have
\begin{equation}
	\begin{split}
	\label{lemma_result}
		\left|h_{c}(\mathbf{x}) - \tilde{h}_{c}(\mathbf{x})\right| \le \sqrt{C}\|\mathbf{h}(\mathbf{x})- \tilde{\mathbf{h}}(\mathbf{x})\|_2  \leq  \epsilon B\sqrt{C\left(\tilde{\lambda}^2 - \xi\right)}.
	\end{split}
\end{equation}
Inserting back $\mathbf{z} = \mathbf{x} + \bm\delta$ and $\epsilon = \|\bm\delta\|_2$ into Eq. (\ref{lemma_result}), we have 
\begin{equation}
	\left|h_{c}(\mathbf{x}) - h_c(\mathbf{z})\right|  \leq \|\mathbf{x}-\mathbf{z} \|_2 B\sqrt{C\left(\tilde{\lambda}^2 - \xi\right)}.
\end{equation}
This completes the proof.

\end{proof}

%% file: supp_theorem.tex
{
	\theoremstyle{plain}
	\newtheorem*{assumption1f}{Assumption 4.4}
}

\subsection{Proof of Theorem \ref{main_theorem}}

\textbf{Single Classifier.} We analyse the expected 0/1 risk of  a single acceptable classifier $\mathbf{h}\in \mathcal{H}$ for a small dataset $D_2 = \{({\mathbf{x}}_i,  y_i ), ({\mathbf{x}}_j,  y_j)\}$ containing the two examples from the ambiguous pair $a= ((\mathbf{x}_i,  y_i), ({\mathbf{x}}_j, y_j))$. The risk is expressed by
\begin{equation}
	\mathbb{E}_{\mathbf{h}\in\mathcal{H}}[\hat{\mathcal{R}}_{0/1}(D_2,\mathbf{h})] = \mathbb{E}_{\mathbf{h}\in\mathcal{H}}\left[ \frac{1}{2}\left({1}\left[h_{y_i}({\mathbf{x}}_i)<\max_{c\ne y_i}h_{c}({\mathbf{x}}_i)\right]+{1}\left[h_{y_j}({\mathbf{x}}_j)<\max_{d\ne y_j}h_{c}({\mathbf{x}}_j)\right]\right)\right]. 
	\label{eq_thm_h_risk}
\end{equation}
We consider three cases.

\noindent
\textbf{Case I}: Suppose the example $({\mathbf{x}}_i,  y_i )$ is correctly classified, thus, according to Assumption \ref{assu_simple_classification} for acceptable classifiers, it has $h_{y_i}({\mathbf{x}}_i)\ge 0.5+\frac{1}{J}$. As a result, its prediction score for a wrong class  ($c\neq y_i$) satisfies 
\begin{equation}
	\label{eq_thm_h_r1}
	h_{c}({\mathbf{x}}_i)\le 1 - h_{y_i}({\mathbf{x}}_i)\le1-(0.5+\frac{1}{J}) =0.5-\frac{1}{J}<0.5	< h_{y_i}({\mathbf{x}}_i).
\end{equation}
Applying Lemma \ref{lem_2} for $c=y_i$ and Eq. (\ref{eq_ambig_cond}) in Definition \ref{ass_borderline_data} for ambiguous pair, it has  
\begin{equation}
	\label{eq_loss_cur}
	h_{y_i}({\mathbf{x}}_i) - h_{y_i}({\mathbf{x}}_j)\leq |h_{y_i}({\mathbf{x}}_i) - h_{y_i}({\mathbf{x}}_j)| \leq  \|\mathbf{x}_i-\mathbf{x}_j \|_2 B\sqrt{C\left(\tilde{\lambda}^2 - \xi\right)} \leq \frac{1}{J}.
\end{equation}
Combining the above with the Case I assumption of $h_{y_i}({\mathbf{x}}_i)\ge 0.5+\frac{1}{J}$, it has
\begin{equation}
	\label{eq_drop_half}
	h_{y_i}({\mathbf{x}}_j) \ge h_{y_i}({\mathbf{x}}_i) - \frac{1}{J}\ge(0.5 + \frac{1}{J}) - \frac{1}{J} = 0.5, 
\end{equation}
and hence, for any $c\ne y_i$,  it has
\begin{equation}
	\label{eq_thm_h_r2} 
	h_{c}({\mathbf{x}}_j) < 1 - 	h_{y_i}({\mathbf{x}}_j) \le 0.5\le h_{y_i}({\mathbf{x}}_j),
\end{equation}
which indicates that the example $({\mathbf{x}}_j, y_j)$ is wrongly predicted to class $y_i$ in Case I. Therefore, 
\begin{equation}
	\label{eq_thm_single_risk_case1}
	\hat{\mathcal{R}}^{(\textmd{I})}_{0/1}(D_2,\mathbf{h})=\frac{0+1}{2}=\frac{1}{2}.
\end{equation}

\noindent
\textbf{Case II}:  Suppose the example $({\mathbf{x}}_j,  y_j )$ is correctly classified.  Following exactly the same derivation as in Case I, this results in the wrong classification of the other example $({\mathbf{x}}_i,  y_i)$ into class $y_j$. Therefore,  
\begin{equation}
	\label{eq_thm_single_risk_case2}
	\hat{\mathcal{R}}^{(\textmd{II})}_{0/1}(D_2,\mathbf{h})=\frac{1+0}{2}=\frac{1}{2}.
\end{equation}

\noindent
\textbf{Case III}: Suppose both examples are misclassified, which simply results in 
\begin{equation}
	\label{eq_thm_single_risk_case3}
	\hat{\mathcal{R}}^{(\textmd{III})}_{0/1}(D_2,\mathbf{h})=\frac{1+1}{2}=1.
\end{equation}

\noindent
Note that these three cases are mutually exclusive. Use $E_1$, $E_2$ and $E_3$ to represent the three events corresponding to Case I, Case II and Case III, respectively.
Letting $p$ denote the probability of correctly classifying  an example by an acceptable classifier, it is straightforward to obtain $p(E_3) = (1-p)^2$, while $p(E_1) =p(E_2) = \frac{1}{2}\left(1-(1-p)^2\right) = p-\frac{1}{2}p^2$. 
Therefore, it has 
\begin{align}
	\label{eq_prob_single}
	& \mathbb{E}_{\mathbf{h}\in\mathcal{H}}\left[\hat{\mathcal{R}}_{0/1}(D_2,\mathbf{h})\right] \\
	\nonumber
	=\;& \hat{\mathcal{R}}^{(\textmd{I})}_{0/1}(D_2,\mathbf{h}) p(E_1) +\hat{\mathcal{R}}^{(\textmd{II})}_{0/1}(D_2,\mathbf{h})p(E_2) + \hat{\mathcal{R}}^{(\textmd{III})}_{0/1}(D_2,\mathbf{h})p(E_3),\\
	\nonumber
	=\; & \frac{1}{2} p(E_1) + \frac{1}{2} p(E_2) +p(E_3) =  p-\frac{1}{2}p^2 + (1-p)^2 =1-p+\frac{1}{2}p^2.
\end{align}

\textbf{Ensemble Classifier.} We next analyse using  $D_2$ the expected 0/1 risk  of an ensemble of two acceptable base classifiers ($\mathbf{h}^0, \mathbf{h}^1\in \mathcal{H}$) with a $\emph{max}$ or average combiner, in five cases.

\noindent
\textbf{Case I}: Suppose the example $({\mathbf{x}}_i,  y_i )$ is correctly classified by both  base classifiers. According to Assumption \ref{assu_simple_classification} for acceptable classifiers, it has $h^0_{y_i}({\mathbf{x}}_i)\ge 0.5+\frac{1}{J}$ and $h^1_{y_i}({\mathbf{x}}_i)\ge0.5+\frac{1}{J}$. Following exactly the same derivation as in the earlier Case I analysis for a single classifier, i.e., Eqs. (\ref{eq_thm_h_r1}) and (\ref{eq_thm_h_r2}), the following holds for any $c\neq y_i$, as
\begin{align}
	\label{eq_case1_1}
	&h^0_{c}({\mathbf{x}}_i) < h^0_{y_i}({\mathbf{x}}_i),\  h^0_{c}({\mathbf{x}}_j) < h^0_{y_i}({\mathbf{x}}_j), \\
	\label{eq_case1_2}
	&h^1_{c}({\mathbf{x}}_i) < h^1_{y_i}({\mathbf{x}}_i),\  h^1_{c}({\mathbf{x}}_j) < h^1_{y_i}({\mathbf{x}}_j).
\end{align}
As a result, for any $c\neq y_i$, the ensemble prediction  satisfies the following
\begin{align}
	&h_{e, y_i}^{(0,1)}({\mathbf{x}}_i) = \max\left(h^0_{y_i}({\mathbf{x}}_i), h^1_{y_i}({\mathbf{x}}_i) \right) >  \max(h^0_{c}({\mathbf{x}}_i),h^1_{c}({\mathbf{x}}_i)) = h_{e, c}^{(0,1)}({\mathbf{x}}_i), \\
	&h_{e, y_i}^{(0,1)}({\mathbf{x}}_i) = \frac{1}{2}\left(h^0_{y_i}({\mathbf{x}}_i) + h^1_{y_i}({\mathbf{x}}_i) \right) >  \frac{1}{2}(h^0_{c}({\mathbf{x}}_i) + h^1_{c}({\mathbf{x}}_i)) = h_{e, c}^{(0,1)}({\mathbf{x}}_i),
\end{align}
each corresponding to the $\max$ and average combiners, respectively.
This indicates a correct ensemble classification of $({\mathbf{x}}_i, y_i)$. Also, it satisfies
\begin{align}
	&h_{e, y_j}^{(0,1)}({\mathbf{x}}_j) =  \max\left(h^0_{y_j}({\mathbf{x}}_j), h^1_{y_j}({\mathbf{x}}_j) \right) <  \max(h^0_{y_i}({\mathbf{x}}_j),h^1_{y_i}({\mathbf{x}}_j)) = h_{e, y_i}^{(0,1)}({\mathbf{x}}_j),  \\
	&h_{e, y_j}^{(0,1)}({\mathbf{x}}_j) = \frac{1}{2}\left(h^0_{y_j}({\mathbf{x}}_j) + h^1_{y_j}({\mathbf{x}}_j) \right) <  \frac{1}{2}(h^0_{y_i}({\mathbf{x}}_j) + h^1_{y_i}({\mathbf{x}}_j)) = h_{e, y_i}^{(0,1)}({\mathbf{x}}_j),  
\end{align}
when using the $\max$ and average combiners, respectively. This indicates a wrong classification of $({\mathbf{x}}_j, y_j)$.
Finally, for Case I, we have
\begin{equation}
	\label{eq_thm_case1}
	\hat{\mathcal{R}}^{(\textmd{I})}_{0/1}\left(D_2, \mathbf{h}_{e}^{(0,1)}\right)=\frac{1}{2}\left(0 + 1\right)=\frac{1}{2},
\end{equation}

\noindent
\textbf{Case II}: Suppose the example $({\mathbf{x}}_j,  y_j )$ is correctly classified by both   base classifiers.  By following exactly the same derivation as in Case I as above,  the ensemble correctly classifies $({\mathbf{x}}_j,  y_j )$, while wrongly classifies $({\mathbf{x}}_i,  y_i)$. As a result, it has
\begin{equation}
	\label{eq_thm_case2}
	\hat{\mathcal{R}}^{(\textmd{II})}_{0/1}\left(D_2, \mathbf{h}_{e}^{(0,1)}\right)=\frac{1}{2}\left(1 + 0\right)=\frac{1}{2}.
\end{equation}

\noindent
\textbf{Case III}: Suppose the example $({\mathbf{x}}_i,  y_i )$ is correctly classified  by $\mathbf{h}^0$, while the other example $({\mathbf{x}}_j,  y_j)$ is correctly classified  by $\mathbf{h}^1$, i.e.,  $h^0_{y_i}({\mathbf{x}}_i)\ge0.5+\frac{1}{J}$ and $h^1_{y_j}({\mathbf{x}}_j)\ge0.5+\frac{1}{J}$ according to Assumption \ref{assu_simple_classification}. 
Following a similar analysis as in Case I for a single classifier, we know that $\mathbf{h}^0$ consequently misclassifies $({\mathbf{x}}_j,  y_j)$ into $y_i$, while $\mathbf{h}^1$ misclassifies $({\mathbf{x}}_i,  y_i )$ into $y_j$. 
Also, by Assumption \ref{assu_simple_classification}, it is assumed that the misclassification happens with a less score than $0.5+\frac{1}{J}$, thus, $h^0_{y_i}(\mathbf{x}_j)\le 0.5+\frac{1}{J}$ and $h^1_{y_j}(\mathbf{x}_i)\le 0.5+\frac{1}{J}$.  
Combining all these,  for any $c\ne y_i$ and $d\ne y_j$, we have
\begin{align} 
	\label{eq_thm_case3_1}
	&h^1_d(\mathbf{x}_i)<0.5\le h^1_{y_j}(\mathbf{x}_i)\le 0.5+\frac{1}{J}\le h^0_{y_i}(\mathbf{x}_i),\\
	\label{eq_thm_case3_2}
	&h^0_c(\mathbf{x}_j)<0.5\le h^0_{y_i}(\mathbf{x}_j)\le 0.5+\frac{1}{J}\le h^1_{y_j}(\mathbf{x}_j),
\end{align}
and according to the second condition in Assumption \ref{assu_simple_classification}, it has 
\begin{align}
	\label{eq_assum_ineq_i_1}	
	h^0_{c}({\mathbf{x}}_i)\le &\frac{1-h^0_{y_i}({\mathbf{x}}_i)}{C-1}\le h^0_{y_i}({\mathbf{x}}_i), \\
	\label{eq_assum_ineq_j_2}	
	h^1_{d}({\mathbf{x}}_j)\le& \frac{1-h^1_{y_j}({\mathbf{x}}_j)}{C-1}\le h^1_{y_j}({\mathbf{x}}_j).
\end{align} 
Subsequently, the ensemble prediction by a $\max$ combiner satisfies
\begin{align}
	h_{e, y_i}^{(0,1)}({\mathbf{x}}_i) =& \max\left(h^0_{y_i}({\mathbf{x}}_i), h^1_{y_i}({\mathbf{x}}_i) \right) = h^0_{y_i}({\mathbf{x}}_i) >  \max(h^0_{c}({\mathbf{x}}_i),h^1_{c}({\mathbf{x}}_i)) = h_{e, c}^{(0,1)}({\mathbf{x}}_i), \\
	h_{e, y_j}^{(0,1)}({\mathbf{x}}_j) =&  \max\left(h^0_{y_j}({\mathbf{x}}_j), h^1_{y_j}({\mathbf{x}}_j) \right) = h^1_{y_j}({\mathbf{x}}_j) >  \max(h^0_{d}({\mathbf{x}}_j),h^1_{d}({\mathbf{x}}_j))   = h_{e, d}^{(0,1)}({\mathbf{x}}_j),
\end{align}
which indicates a correct classification of both examples.

Now we consider the slightly more complex situation of ensemble by averaging.  According to the previous analysis, we know that $\mathbf{x}_i$ is classified by $\mathbf{h}^1$ to $y_j$, and $\mathbf{x}_i$ is classified by $\mathbf{h}^0$ to $y_i$. Applying the second condition in Assumption \ref{assu_simple_classification},  we analyse the quantity  $1-h^1_{y_j}({\mathbf{x}}_i)-h^1_{y_i}({\mathbf{x}}_i)$ as 
\begin{equation} 
	\label{eq_assum_ineq_i_2_tmp}
	\small
	1-h^1_{y_j}({\mathbf{x}}_i)-h^1_{y_i}({\mathbf{x}}_i)=\sum_{c\ne y_i, y_j} h^1_{c}({\mathbf{x}}_i)\le (C-2)\frac{1-h^1_{y_j}({\mathbf{x}}_i)}{C-1}=1-h^1_{y_j}({\mathbf{x}}_i) -\left(\frac{1-h^1_{y_j}({\mathbf{x}}_i)}{C-1}\right),
\end{equation}
resulting in 
\begin{equation} 
	\label{eq_assum_ineq_i_2}
	h^1_{y_i}({\mathbf{x}}_i)\ge\frac{1-h^1_{y_j}({\mathbf{x}}_i)}{C-1}.
\end{equation}
Combining Eq. (\ref{eq_thm_case3_1}), Eq. (\ref{eq_assum_ineq_i_1}) and Eq. (\ref{eq_assum_ineq_i_2}), it has
\begin{equation}
	\label{eq_sum_ineq_i_1}
	h^1_{y_i}({\mathbf{x}}_i)\ge  \frac{1-h^1_{y_j}({\mathbf{x}}_i)}{C-1}> \frac{1-h^0_{y_i}({\mathbf{x}}_i)}{C-1}\ge h^0_{c}({\mathbf{x}}_i).
\end{equation}
On the other hand, from Eq. (\ref{eq_thm_case3_1}), one can obtain 
\begin{equation} 
	\label{eq_thm_si3_max}
	h^0_{y_i}({\mathbf{x}}_i)\ge   h^1_{c}({\mathbf{x}}_i).
\end{equation}
As a result, the ensemble prediction by an average combiner satisfies 
\begin{equation}
	\label{case3_res1}
	h_{e, y_i}^{(0,1)}({\mathbf{x}}_i) = \frac{1}{2}\left(h^0_{y_i}({\mathbf{x}}_i) + h^1_{y_i}({\mathbf{x}}_i) \right) >  \frac{1}{2}(h^0_{c}\left({\mathbf{x}}_i) + h^1_{c}({\mathbf{x}}_i) \right) = h_{e, c}^{(0,1)}({\mathbf{x}}_i),
\end{equation}
for any $c\neq y_i$.
Following the same way of deriving Eqs. (\ref{eq_sum_ineq_i_1}) and (\ref{eq_thm_si3_max}), but for $\mathbf{x}_j$, we can obtain another two inequalities  $h^1_{y_j}({\mathbf{x}}_j) \geq h^0_{d}({\mathbf{x}}_j)$ and  $h^0_{y_j}({\mathbf{x}}_j) \ge   h^1_{d}({\mathbf{x}}_j)$, for any $d\neq y_j$, and subsequently, \begin{equation}
	\label{case3_res2}
	h_{e, y_j}^{(0,1)}({\mathbf{x}}_j) =  \frac{1}{2}\left(h^0_{y_j}({\mathbf{x}}_j) + h^1_{y_j}({\mathbf{x}}_j) \right) >  \frac{1}{2}\left(h^0_{d}({\mathbf{x}}_j) + h^1_{d}({\mathbf{x}}_j)\right) = h_{e, d}^{(0,1)}({\mathbf{x}}_j).
\end{equation}
Putting together Eqs. (\ref{case3_res1}) and (\ref{case3_res2}),  a correct ensemble classification is achieved for both examples. Finally, we conclude the following result 
\begin{equation}
	\label{eq_thm_case3}
	\hat{\mathcal{R}}^{(\textmd{III})}_{0/1}\left(D_2, \mathbf{h}_{e}^{(0,1)}\right) =0,
\end{equation}
which is applicable to both the $\max$ and average combiners.

\noindent
\textbf{Case IV}: Suppose the example $({\mathbf{x}}_i,  y_i )$ is correctly classified  by $\mathbf{h}^1$ while the other example $({\mathbf{x}}_j,  y_j)$ is correctly classified  by $\mathbf{h}^0$. This is essentially the same situation as in Case III, and the same result $\hat{\mathcal{R}}^{(\textmd{IV})}_{0/1}\left(D_2, \mathbf{h}_{e}^{(0,1)}\right) =0 $ is obtained. 

\noindent
\textbf{Case V}: This case represents all the remaining situations, where, for instance,  the example $({\mathbf{x}}_i,  y_i )$ and/or $({\mathbf{x}}_i,  y_i )$ is misclassified by both base classifiers. Here, we do not have sufficient information to analyse the error in detail, and also it is not necessary to do so for our purpose. So we just simply leave it as $\hat{\mathcal{R}}^{(\textmd{V})}_{0/1}\left(D_2, \mathbf{h}_{e}^{(0,1)}\right) \leq 1$.

\noindent
These five cases are mutually exclusive, and we use $\{H_i\}_{i=1}^5$ to denote them accordingly.    The first four cases represent the same situation that each example is correctly classified by a single base classifier, therefore $p(H_1) = p(H_2) = p(H_3) = p(H_4) = p(E_1)p(E_2) = \left(p-\frac{1}{2}p^2\right)^2$, while $p(H_5) =1-\sum_{i=1}^4p(H_i) = 1-4\left(p-\frac{1}{2}p^2\right)^2= 1-(2p-p^2)^2$. Incorporating the result of $\hat{\mathcal{R}}_{0/1}\left(D_2, \mathbf{h}_{e}^{(0,1)}\right)$ regarding to the five cases, we have
\begin{align}
	\label{eq_prob_ensemble}
	&\mathbb{E}_{ \mathbf{h}^0,\mathbf{h}^1\in\mathcal{H} }\left[\hat{\mathcal{R}}_{0/1}\left(D_2, \mathbf{h}_{e}^{(0,1)}\right)\right] \\
	\nonumber
	\leq \; & \frac{1}{2} p(H_1) + \frac{1}{2} p(H_2) +0(p(H_3)+p(H_4)) +p(H_5)   \\
	\nonumber
	= \; & \left(p-\frac{1}{2}p^2\right)^2 +1-(2p-p^2)^2 =1-3p^2+3p^3-\frac{3}{4}p^4.
\end{align}

\textbf{Risk Comparison.} We examine the sufficient condition for achieving a reduced ensemble loss for this dataset $D_2$, i.e.,
\begin{equation}
	\label{D2_result}
	\mathbb{E}_{\mathbf{h}^0,\mathbf{h}^1\in\mathcal{H}}\left[\hat{\mathcal{R}}_{0/1}\left(D_2,  \mathbf{h}^{(0,1)}_e \right)\right] <  \mathbb{E}_{\mathbf{h}\in\mathcal{H}}\left[\hat{\mathcal{R}}_{0/1}\left(D_2,\mathbf{h}\right)\right].
\end{equation}
Incorporating Eqs. (\ref{eq_prob_single}) and  (\ref{eq_prob_ensemble}), this requires to solve the following polynomial inequality, as 
\begin{equation}
	1-3p^2+3p^3-\frac{3}{4}p^4< 1-p+\frac{1}{2}p^2,
\end{equation}
for which $p>0.425$ provides a solution.  Applying  the expectation  $\mathbb{E}_{a\sim A(D)}$ over the data samples, where the ambiguous pair $a$ is equivalent to $D_2$, Eq. (\ref{eq_thm_1}) from the theorem is obtained. This completes the proof.

%% file: supp_Toy.tex
\section{A Toy Example for Theorem \ref{main_theorem}} \label{toy}

\begin{figure}[h]
	\centering
	\includegraphics[width=0.45\textwidth]{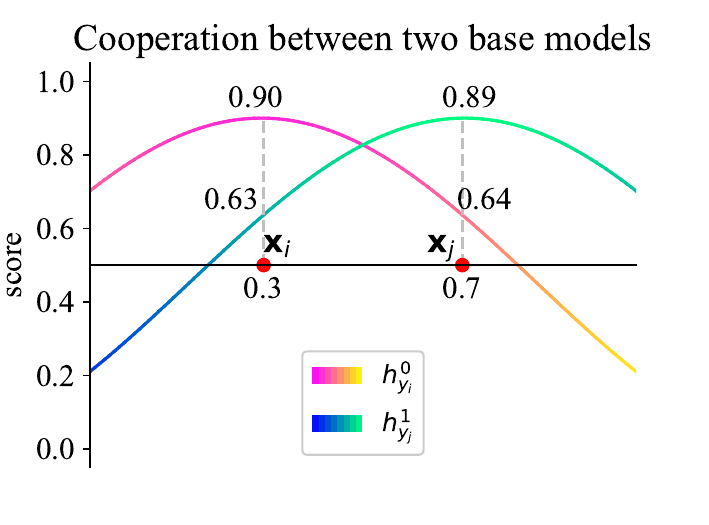}
	\caption{\label{fig_exmaple_thm1} Illustration for Theorem \ref{main_theorem}.}
\end{figure}

Guided by Theorem \ref{main_theorem}, we aim to verify the difference between the single-branch and  ensemble mechanisms using 1-dimensional 2-class data. Suppose $D=\{(\mathbf{x}_i=0.3, y_i=0), (\mathbf{x}_j=0.7, y_j=1)\}$, we use it to construct an ambiguous pair $a=((\mathbf{x}_i, y_i), (\mathbf{x}_j, y_j))$ as presented in Fig. \ref{fig_exmaple_thm1}. We select two base models $\mathbf{h}^0, \mathbf{h}^1\in\mathcal{H}$ such that $\mathbf{h}^0$ classifies 
$\mathbf{x}_i$ well and $\mathbf{h}^1$ classifies $\mathbf{x}_j$ well. W.l.o.g, let $\mathbf{h}=\mathbf{h}^0$.
For the classifiers $\mathbf{h}$, $\mathbf{h}^0$ and $\mathbf{h}^1$, we analyze the 0-1 loss defined in Eq. (\ref{eq_01_loss}). Then, we have

\begin{align}
	\nonumber
	\hat{\mathcal{R}}_{0/1}(a,\max(\mathbf{h}^0,\mathbf{h}^1))=\; & \frac{1}{2}\big({1} \left[\max(0.9, 0.37)<\max(0.1, 0.63)\right]\\
	\nonumber
	&+{1} \left[\max(0.89, 0.36)<\max(0.11, 0.64)\right]\big) =0,\\
	\nonumber
	\hat{\mathcal{R}}_{0/1}(a,\left(\mathbf{h}^0+\mathbf{h}^1\right)/2)=\; & \frac{1}{2}\big({1} \left[(0.9+ 0.37)/2<(0.1, 0.63)/2\right]\\
	\nonumber
	&+{1} \left[(0.89+ 0.36)/2<(0.11+ 0.64)/2\right]\big) =0,\\
	\nonumber
	\hat{\mathcal{R}}_{0/1}(a,\mathbf{h})=\;  & \frac{1}{2}\big({1}  \left[0.9<0.63\right]+{1} \left[0.36<0.64\right]\big) =0.5.
\end{align}
Hence, it has  $\hat{\mathcal{R}}_{0/1}(a,\max(\mathbf{h}^0,\mathbf{h}^1))<\hat{\mathcal{R}}_{0/1}(a,\mathbf{h})$ and $\hat{\mathcal{R}}_{0/1}(a,\left(\mathbf{h}^0+\mathbf{h}^1\right)/2)<\hat{\mathcal{R}}_{0/1}(a,\mathbf{h})$, which matches the resulting inequality in Theorem \ref{main_theorem}.

%% file: supp_ExtraResults.tex
\begin{table}[t]
	\caption{Results on two additional black-box attacks.}
	\label{tabtwo_black_box}
	\begin{center}
			\begin{tabular}{l|cc}
				\toprule
				& Simple Attack (\%) & Bandits Attack (\%) \\
				\midrule 
				\midrule
				ADP & $75.91$ & $59.21$\\
				$\textmd{iGAT}_{\textmd{ADP}}$ & $\textbf{79.43}$ & $\textbf{64.55}$\\
				\midrule
				DVERGE & $79.43$ & $63.80$\\
				$\textmd{iGAT}_{\textmd{DVERGE}}$ & $\textbf{79.61}$ & $\textbf{64.89}$\\
				\midrule
				CLDL & $76.82$ & $63.80$\\
				$\textmd{iGAT}_{\textmd{CLDL}}$ & $\textbf{78.84}$ &  $\textbf{65.25}$\\
				\midrule
				SoE & $\textbf{76.22}$& $66.10$\\
				$\textmd{iGAT}_{\textmd{SoE}}$  & $75.18$& $\textbf{66.50}$\\
				\bottomrule
			\end{tabular}
	\end{center}
	\vskip -0.1in
\end{table}

\begin{table}[t]
	\caption{Comparison of the ensemble robustness (\%) to adversarial attacks of various perturbation strengths, using the AutoAttack on CIFAR-10. The results are averaged over five independent runs.}
	\label{tab_test_various_strengths}
	\vskip 0.15in
	\begin{center}
			\begin{tabular}{l|r|ccccc}
					\toprule
					\multicolumn{2}{c|}{\backslashbox[10mm]{}{$\epsilon$}} & { $2/255$}  & { $4/255$} & { $6/255$} & {$8/255$} & {$10/255$} \\
					\midrule
					\midrule
					\multirow{3}{*}{\rotatebox[origin=c]{90}{CIFAR10}} 
					& CLDL & $71.16$ & $60.36$ & $48.89$ & $37.06$ & $26.00$ \\	
					&   $\textmd{iGAT}_{\textmd{CLDL}}$  & $\textbf{72.69}$ &  $\mathbf{61.19}$ & $\textbf{49.07}$ & $\textbf{37.12}$ & $25.96$   \\
					\cmidrule(r){2-7}
					&  DVERGE  & $76.01$ &  $64.80$ & $51.92$ & $39.22$ & $27.72$   \\
					&   $\textmd{iGAT}_{\textmd{DVERGE}}$  & $.\textbf{76.19}$ &  $\textbf{65.14}$ & $\textbf{52.52}$ & $\textbf{39.48}$ & $\textbf{28.59}$   \\
					\cmidrule(r){2-7}
					& ADP & $71.93$ &  $59.53$ & $47.27$ & $35.52$ & $25.01$  \\
					&   $\textmd{iGAT}_{\textmd{ADP}}$  & $\textbf{76.02}$ &  $\textbf{64.76}$ & $\textbf{52.44}$ & $\textbf{40.38}$ & $\textbf{29.46}$   \\
					\midrule
					\midrule
					\multirow{3}{*}{\rotatebox[origin=c]{90}{CIFAR100}}
					&  SoE  & $46.55$ & $33.89$ & $23.77$ & $15.92$ & $10.49$\\
					&  $\textmd{iGAT}_{\textmd{SoE}}$  & $45.72$ & $33.18$ & $23.28$ & $\textbf{16.09}$ & $\textbf{10.52}$\\
					\cmidrule(r){2-7}
					& DVERGE  & $48.87$ &  $35.81$ & $25.35$ & $17.26$ & $11.18$  \\
					& $\textmd{iGAT}_{\textmd{DVERGE}}$  & $\textbf{49.43}$ &  $\textbf{37.11}$ & $\textbf{26.78}$ & $\textbf{18.60}$ & $\textbf{12.13}$  \\
					\cmidrule(r){2-7}					
					& ADP   &  $45.67$ &  $33.90$ & $24.42$ &  $17.36$ & $12.27$   \\
					& $\textmd{iGAT}_{\textmd{ADP}}$    &  $\textbf{46.33}$ &  $\textbf{34.33}$ & $\textbf{24.85}$ &  $\textbf{17.86}$ & $\textbf{12.53}$    \\
					\bottomrule
				\end{tabular}
	\end{center}
	\vskip -0.1in
\end{table}

\begin{table}[t]
	\caption{Comparison between  $\textmd{iGAT}_{\textmd{ADP}}$ (average combiner) and a baseline single classifier,  evaluated using  CIFAR-10  data and the PGD attack ($\epsilon=8/255$). The results are averaged over five independent runs.}
	\label{tab_ensem2single}
	\begin{center}
			\begin{tabular}{l|c|c|c}
				\toprule
				& Natural (\%)  & PGD (\%) & Model size \\
				\midrule 
				Single Classifier & $81.23$ & $38.33$ & 43M\\
				$\textmd{iGAT}_{\textmd{ADP}}$   & $\textbf{84.95}$ & $\textbf{46.25}$  & 9M \\
				\bottomrule
			\end{tabular}
	\end{center}
	\vskip -0.1in
\end{table}

\begin{table}[!h]
	\caption{Probabilities of base models classifying correctly adversarial examples from the CIFAR-10.}
	\label{tab_probability}
	\begin{center}
			\begin{tabular}{l|cccccc}
				\toprule
				& ADP & $\textmd{iGAT}_{\textmd{ADP}}$ & DVERGE & $\textmd{iGAT}_{\textmd{DVERGE}}$ & CLDL & $\textmd{iGAT}_{\textmd{CLDL}}$ \\ 
				\midrule
				$p$ & $41.92\%$ & $45.98\%$ & $46.25\%$ & $47.82\%$ & $50.37\%$ & $51.02\%$ \\
				\bottomrule
			\end{tabular}
	\end{center}
	\vskip -0.1in
\end{table}

\begin{table}[!h]
	\caption{Distributions of predicted scores by base models correctly classifying adversarial examples from the CIFAR-10.}
	\label{tab_probability_distribution}
	\begin{center}
			\begin{tabular}{l|cccccc}
				\toprule
				Interval & $<$0.5 & 0.5-0.6 & 0.6-0.7 & 0.7-0.8 & 0.8-0.9  & 0.9-1.0\\ 
				\midrule
				$\textmd{iGAT}_{\textmd{ADP}}$ & $43.55\%$ & $13.30\%$ & $11.15\%$ & $10.20\%$ & $10.62\%$ & $11.19\%$ \\
				$\textmd{iGAT}_{\textmd{DVERGE}}$ & $20.07\%$ & $13.15\%$ & $12.20\%$ & $12.46\%$ & $14.44\%$ & $27.69\%$ \\
				$\textmd{iGAT}_{\textmd{CLDL}}$ & $49.50\%$ & $14.12\%$ & $12.26\%$ & $13.53\%$ & $9.77\%$ & $0.81\%$ \\
				\bottomrule
			\end{tabular}
	\end{center}
	\vskip -0.1in
\end{table}

\begin{table}[!h]
	\caption{Expectations of the maximum predicted scores on incorrect classes among base models when tested on adversarial examples from the CIFAR-10.}
	\label{tab_most_wrong_score}
	\begin{center}
			\begin{tabular}{cccccc}
				\toprule
			 	ADP & $\textmd{iGAT}_{\textmd{ADP}}$ & DVERGE & $\textmd{iGAT}_{\textmd{DVERGE}}$ & CLDL & $\textmd{iGAT}_{\textmd{CLDL}}$ \\ 
				\midrule
				$0.390$ & $0.323$ & $0.476$ & $0.396$ & $0.320$ & $0.281$ \\
				\bottomrule
			\end{tabular}
	\end{center}
	\vskip -0.1in
\end{table}

\begin{table}[t]
	\caption{Probabilities of $h^i_{y_\mathbf{x}}({\mathbf{x}})\ge \frac{1-h^i_{\hat{y}}({\mathbf{x}})}{C-1}$ for $y_{\mathbf{x}}\ne \hat{y}$ when tested on adversarial examples from the CIFAR-10.}
	\label{tab_above_line_score}
	\begin{center}
			\begin{tabular}{cccccc}
				\toprule
				ADP & $\textmd{iGAT}_{\textmd{ADP}}$ & DVERGE & $\textmd{iGAT}_{\textmd{DVERGE}}$ & CLDL & $\textmd{iGAT}_{\textmd{CLDL}}$ \\ 
				\midrule
				$68.74\%$ & $73.12\%$ & $78.99\%$ & $80.19\%$ & $78.39\%$ & $80.87\%$ \\
				\bottomrule
			\end{tabular}
	\end{center}
	\vskip -0.1in
\end{table}

\begin{table}[!h]
	\caption{Predicted scores on incorrectly classified adversarial examples by the best-performing base model using the CIFAR-10.}
	\label{tab_increased_score}
	\begin{center}
			\begin{tabular}{cccccc}
				\toprule
				ADP & $\textmd{iGAT}_{\textmd{ADP}}$ & DVERGE & $\textmd{iGAT}_{\textmd{DVERGE}}$ & CLDL & $\textmd{iGAT}_{\textmd{CLDL}}$ \\ 
				\midrule
				$0.264$ & $0.291$ & $0.231$ & $0.240$ & $0.235$ & $0.241$ \\
				\bottomrule
			\end{tabular}
	\end{center}
	\vskip -0.1in
\end{table}

\section{Additional Experiments and Results}

\textbf{Extra Black-box Attacks:}
We conduct more experiments to test the effectiveness of iGAT, by evaluating against another two  time-efficient and commonly used black-box attacks, using the CIFAR-10 dataset.  Results are reported in Table \ref{tabtwo_black_box}. It can be seen that, in most cases, a  robustness improvement has been achieved by the enhanced  defence.

\textbf{Varying Perturbation Strengths:}
In addition to the perturbation strength $\epsilon=8/255$ tested in the main experiment, we compare the defense techniques under AutoAttack with different settings of perturbation strength.  
Table \ref{tab_test_various_strengths} reports the resulting classification accuraccies,
demonstrating that the proposed  iGAT is able to improve the adversarial robustness of  the studied  defense techniques in most cases.

\textbf{Comparison Against Single Classifiers:}
To observe how an ensemble classifier performs with specialized ensemble adversarial training, we compare $\textmd{iGAT}_{\textmd{ADP}}$ based on the average combiner against a single-branch classifier. 
This classifier uses the ResNet-18 architecture, and is trained using only the standard adversarial training without any diversity or regularization driven treatment. 
Table \ref{tab_ensem2single} reports the results. 
It can be  seen that the specialized  ensemble adversarial training technique can significantly improve both the natural accuracy and adversarial robustness.

\textbf{Experiments Driven by Assumption \ref{assu_simple_classification}:}
To approximate empirically the probability $p$ that a trained base classifier can correctly classify a challenging example, we generate a set of globally adversarial examples $\tilde{\mathbf{X}}$ by attacking the ensemble $\mathbf{h}$ (average combiner) using the PGD and then estimate $p$ on this dataset by $p=\mathbb{E}_{{i\in[N]},{(\mathbf{x},y_\mathbf{x})\sim \left(\tilde{\mathbf{X}}, \mathbf{y}\right)}}1[h_{y_{\mathbf{x}}}^i(\mathbf{x}) >\max_{c\ne y_{\mathbf{x}}}h_c^i(\mathbf{x})]$. From Table \ref{tab_probability}, we can see that all the enhanced ensembles contain base models with a higher probability for correct classifications.

We then examine the distributions of predicted scores by base models when classifying correctly the globally adversarial data generated in the same as in Table \ref{tab_probability}. It can be seen that the case exists, where a base model correctly classifies a challenging example with a sufficiently large predicted score.

Next, we compute the quantity, i.e., the largest incorrectly predicted score  $\mathbb{E}_{i\in[N],\mathbf{(\mathbf{x},y_\mathbf{x})\sim(\tilde{\mathbf{X}}, \mathbf{y})}} \max_{c\ne y_{\mathbf{x}}}h_c^i(\mathbf{x})$, to indirectly estimate whether the small-incorrect-prediction condition, i.e., $f_{c}(\mathbf{x})\leq \frac{1- f_{\hat{y}}(\mathbf{x})}{C-1}$ in Assumption \ref{assu_simple_classification}, can be satisfied better after enhancement. Note that $y_i\ne \hat{y}_i$ indicates the incorrect classification while $y_i =\hat{y}_i$ indicates the opposite, both of which are uniformly measured by the defined quantity. This quantity, which is expected to be small, can also be used to evaluate the effect of the proposed regularization term in Eq. (\ref{eq_regularization}) on the training. Table \ref{tab_most_wrong_score}  shows that the largest wrongly predicted scores by the base models have significantly dropped for all the enhanced ensemble models.

Note that  small values of $h^i_{c\ne y_\mathbf{x}}(\mathbf{x})$ is equivalent to the high values of $h^i_{y_\mathbf{x}}(\mathbf{x})$, and in the theorem,  when $\hat{y} \ne y_{\mathbf{x}}$, $h^i_{y_\mathbf{x}}({\mathbf{x}})\ge\frac{1-h^i_{\hat{y}}({\mathbf{x}})}{C-1}$ is the actual condition expected to be satisfied. Therefore, to examine the second item (the case of misclassification) in Assumption \ref{assu_simple_classification}, we measure the probability $\mathbb{E}_{{i\in[N]},{(\mathbf{x},y_\mathbf{x})\sim(\tilde{\mathbf{X}}, \mathbf{y})}}1\left[h^i_{y_\mathbf{x}}({\mathbf{x}})\ge\frac{1-h^i_{\hat{y}}({\mathbf{x}})}{C-1}\right]$ instead. Table \ref{tab_above_line_score} shows that after enhancement, the probability of satisfying the condition increases.

As shown in Figure \ref{fig_exmaple_thm1}, as long as the peaks of two curves are above the line $x=0.5$ and at similar heights (in which case, are 0.89 and 0.90), whether their height are changed slightly to a higher or lower position will not increase the 0-1 loss. Elevating the low predicted scores to the same level as the high scores serves the crucial factor in fulfilling the cooperative function. Hence, we choose to examine the effect of our distributing rule by checking whether the predicted scores by the best-performing base models on incorrectly classified examples have been increased after enhancement, using the quantity $\mathbb{E}_{(\mathbf{x}, y_\mathbf{x})\sim(\tilde{\mathbf{X}}, \mathbf{y}), \hat{y}_\mathbf{h}(\mathbf{x})\ne y_\mathbf{x}}[\max_{i\in[N]} h_{y_\mathbf{x}}^i(\mathbf{x})]$. It can be seen from Table \ref{tab_increased_score} that  base models were kept improved  on the examples they are already good at classifying.

\textbf{Time Efficiency of iGAT:} (1) On distributing rule: We expect the distributing rule to reduce the  training data size to $\frac{1}{N}$ for training each  base classifier, where  $N$ is the number of base classifiers, and therefore to improve the training time.   We  add  an experiment by comparing the training time on $N=1000$ training samples required by a full version of $\textmd{iGAT}_{\textmd{ADP}}$ and that by a modified version with this distributing rule removed. CIFAR-10 data is used for Evaluation. 
The observed time for $\textmd{iGAT}_{\textmd{ADP}}$ without the distributing design is  $5.63$ seconds, while with the distributing design is  $5.42$ seconds,  indicating a slightly reduced  training time. 
(2) On overall training: We illustrate  the training epochs between  the  ADP defense and its enhancement $\textmd{iGAT}_{\textmd{ADP}}$. ADP necessitates $691$ epochs for ADP,  whereas $\textmd{iGAT}_{\textmd{ADP}}$ only requires   $163$ epochs. Based on these, we can conclude that  $\textmd{iGAT}_{\textmd{ADP}}$ trains faster than ADP.

\textbf{Observation of Curvature:}
We investigated empirically the value of the network curvature $\tilde{\lambda}$ using neural networks trained by the ADP defense techniques, and recorded a $\tilde{\lambda}$ value around $0.06$. The smaller value of $\tilde{\lambda}$  indicates  a looser upper bound in Eq. (\ref{eq_ambig_cond}).  According to our Definition \ref{ass_borderline_data}, a looser upper bound allows to define an ambiguous pair containing two intra-class examples that are less close to each other, thus less challenging to classify.